\documentclass{article} % For LaTeX2e
\usepackage{iclr2025_conference,times}

\usepackage{amsmath,amsfonts,bm}

\def\eqref#1{equation~\ref{#1}}
\def\1{\bm{1}}

\DeclareMathAlphabet{\mathsfit}{\encodingdefault}{\sfdefault}{m}{sl}
\SetMathAlphabet{\mathsfit}{bold}{\encodingdefault}{\sfdefault}{bx}{n}

\usepackage{hyperref}
\usepackage{url}
\usepackage{graphicx}
\usepackage{subcaption}
\usepackage{amsthm}

\newtheorem{theorem}{Theorem}

\title{How Learnable Grids Recover Fine Detail in Low Dimensions: A Neural Tangent Kernel Analysis of Multigrid Parametric Encodings}

\author{Samuel Audia, Soheil Feizi, Matthias Zwicker \& Dinesh Manocha \\
University of Maryland, College Park\\
\texttt{\{sjaudia,sfeizi,zwicker,dmanocha\}@umd.edu} \\
}

\iclrfinalcopy % Uncomment for camera-ready version, but NOT for submission.
\begin{document}

\maketitle

\begin{abstract}

Neural networks that map between low dimensional spaces are ubiquitous in
computer graphics and scientific computing; however, in their naive
implementation, they are unable to learn high frequency information. We present
a comprehensive analysis comparing the two most common techniques for mitigating
this spectral bias: Fourier feature encodings (FFE) and multigrid parametric
encodings (MPE). FFEs are seen as the standard for low dimensional mappings, but
MPEs often outperform them and learn representations with higher resolution and
finer detail. FFE's roots in the Fourier transform, make it susceptible to
aliasing if pushed too far, while MPEs, which use a learned grid structure, have
no such limitation. To understand the difference in performance, we use the
neural tangent kernel (NTK) to evaluate these encodings through the lens of an
analogous kernel regression. By finding a lower bound on the smallest eigenvalue
of the NTK, we prove that MPEs improve a network's performance through the
structure of their grid and not their learnable embedding. This mechanism is
fundamentally different from FFEs, which rely solely on their embedding space to
improve performance. Results are empirically validated on a 2D image regression
task using images taken from 100 synonym sets of ImageNet and 3D implicit
surface regression on objects from the Stanford graphics dataset. Using peak
signal-to-noise ratio (PSNR) and multiscale structural similarity (MS-SSIM) to
evaluate how well fine details are learned, we show that the MPE increases the
minimum eigenvalue by 8 orders of magnitude over the baseline and 2 orders of
magnitude over the FFE. The increase in spectrum corresponds to a 15 dB (PSNR) /
0.65 (MS-SSIM) increase over baseline and a 12 dB (PSNR) / 0.33 (MS-SSIM) increase over the
FFE.

\end{abstract}

\section{Introduction}

Recent advancements in computer graphics, scientific machine learning, and the broader class of implicit neural representations \citep{Essakine2024WhereSurvey} rely on
the simple coordinate-based multi-layer perception (MLP) network, which maps
between low dimensional spaces (typically $\mathbb{R}^1$, $\mathbb{R}^2$, or
$\mathbb{R}^3$). Despite their diminutive size, coordinate-based MLPs have
empowered inverse rendering \citep{Barron2021Mip-NeRF:Fields,
Mildenhall2020NeRF:Synthesis, Muller2022InstantEncoding,
Pumarola2020D-NeRF:Scenes, Yu2020PixelNeRF:Images, Zhang2020NeRF++:Fields},
implicit surface regression \citep{Wang2021SplineFields}, solving the rendering
equation \citep{Hadadan2021NeuralRadiosity}, and enabled Physics-Informed Neural
Networks \citep{Raissi2017PhysicsEquations} in the growing field of scientific
machine learning. These networks need to have a small memory and computational
footprint to maintain the tight timing required in computer
graphics \citep{Muller2021Real-timeTracing} and to allow for computing higher
order derivatives in scientific machine learning. Though simple MLP networks
meet these requirements, in their naive implementation, coordinate-based MLPs
suffer from what is known as the spectral bias problem
\citep{Basri2020FrequencyDensity, Basri2019TheFrequencies,
Wang2020WhenPerspective}, in which they learn higher frequency details orders of
magnitude slower than low frequency details. To understand the cause of this
bias, machine learning practitioners turn to the neural tangent kernel (NTK).

Introduced by \citet{Jacot2018NeuralNetworks}, the neural tangent
kernel describes the training dynamics of a neural network in its infinite width
limit. Though using the NTK directly has been shown to produce worse results
than its finite width counterpart \citep{Li2020LearningNTK}, empirical evaluation of the kernel
for a given data set and architecture allows practitioners to apply classical
techniques, such as eigenvector decomposition, to better understand the
underlying network. In fact, the spectrum of the NTK can be used to show a
network's ability to learn high frequency information in finite width networks
\citep{Tancik2020FourierDomains, Wang2020WhenPerspective, Yang2019ANetworks}. The eigenvectors for
these finer details correspond to the lowest eigenvalues \citep{Basri2019TheFrequencies}, so if network
adaptations are able to raise the spectrum as a whole, the network is better
able to represent the underlying function. For coordinate based MLPs, this
adaption is frequently an encoding that maps the input into a higher dimensional
latent space before being passed to the network.

The two most popular encodings are parametric encodings
\citep{Hadadan2021NeuralRadiosity, Muller2022InstantEncoding} and Fourier
feature encodings (FFE) \citep{Mildenhall2020NeRF:Synthesis,
Tancik2020FourierDomains, Vaswani2017AttentionNeed, Wang2020WhenPerspective}.
Parametric encodings use auxiliary data structures with learnable parameters to
build a higher dimensional embedding space. These data structures frequently
take the form of a grid, with samples being interpolated from fixed grid points
and concatenated. Fourier feature encodings, in contrast, contain no learnable
features and instead use a series of sines and cosines, similar to that of a
Fourier transform, to embed the input into a high dimensional unit hypersphere.
Its ease of implementation and amenability to analysis has made the Fourier
feature encoding the dominant choice in scientific machine learning
\citep{Cuomo2022ScientificNext, Lu2019DeepXDE:Equations,
Raissi2017PhysicsEquations}; however, in computer graphics, parametric encodings
have shown orders of magnitude better performance than Fourier feature encodings
\citep{Hadadan2021NeuralRadiosity, Muller2022InstantEncoding} at the cost of a
larger memory footprint. Both encodings have been extensively evaluated,
ignoring their NTK spectrum, on graphics problems such as 2D image regression,
3D shape regression, inverse rendering, and radiosity calculations. See
\citet{Tancik2020FourierDomains} for the FFE examples and
\citet{Muller2022InstantEncoding,Hadadan2021NeuralRadiosity} for MPE examples.

The FFEs popularity has led to multiple investigations of its effect on the NTK
spectrum \citep{Tancik2020FourierDomains, Wang2020WhenPerspective}. No such
analysis has been known previously for parametric encodings. We remedy this
fact, by providing an in depth analysis of the NTK for parametric encodings.
Specifically, we make the following contributions:

\begin{itemize} 

\item{We derive the neural tangent kernel for the MPE. Through this derivation,
we prove that the multigrid encoding raises the eigenvalue spectrum of the
neural tangent kernel as compared to the baseline coordinate based MLP by
forming a lower bound on its eigenvalues. Our results provide the first
theoretical justification for why MPEs are able to learn finer detail and
discontinuities better than networks with no encoding.} 

\item{We isolate the superior performance of MPEs to their learnable grid and
not their embedding space by evaluating the NTK with and without the
contributions from the grid. This is fundamentally different from FFEs, which
rely solely on the embedding.}

\item{We empirically evaluate the NTK spectrum for the multigrid parametric
encoding, Fourier feature encoding, and baseline identity encoding for 2D image
regression on 100 synonym sets in ImageNet \citep{Deng2009ImageNet:Database} and
3D implicit surface regression \cite{Mescheder2018OccupancySpace} on three meshes
from the Stanford graphics dataset \cite{Turk1994ZipperedImages} to validate our
proof. Peak signal-to-noise ratio (PSNR) and multi-scale structural similarity
index measure (MS-SSIM) \cite{Wang2003Multi-scaleAssessment} measure how well
the learned images learns fine details in the input. The MPE with the smallest
grid cells was shown to increase the minimum eigenvalue 2 orders of magnitude
over the highest frequency FFE and 8 orders of magnitude over the baseline MLP.
This corresponds to an increase of 12 dB (PSNR) / 0.33 (MS-SSIM) and 15 dB
(PSNR) / 0.65 (MS-SSIM), respectively.}

\end{itemize}

\section{Related Work}

Encodings embed a low dimensional input into a higher dimensional space before
passing it to the network. Much like the kernel trick for support vector
machines \citep{10.5555/299094}, the hope is that this new space is then easier
for the network to act on. A key example of this fact is found in the
transformer architecture \citep{Vaswani2017AttentionNeed}. Transformers augment
their input with a positional encoding which gives the attention block
information about an input's location in a larger sequence. The positional
encoding updates inputs with a series of sines and cosines, providing an early
example of the FFE. NeRF \citep{Mildenhall2020NeRF:Synthesis} used the positional
encoding as inspiration to help a coordinate based MLP solve the inverse
rendering problem, in which a scene is learned from images. Though NeRF has been
iterated on, improved, and extended \citep{Barron2021Mip-NeRF:Fields,
Pumarola2020D-NeRF:Scenes, Zhang2020NeRF++:Fields}, the FFE encoding has been a
mainstay. The two most common FFEs are the axis-aligned logarithmic, as explored
in this paper, and the Gaussian FFE \citep{Rahimi2008RandomMachines,
Tancik2020FourierDomains}, which introduces randomness into the encoding.

Multigrid encodings, in contrast, evolved from learned grid
\citep{Fridovich-Keil2021Plenoxels:Networks} and voxel
\citep{Chibane2020ImplicitCompletion} representations. Grid parameters, such as
spherical harmonic coefficients, are learned through backpropagation and
gradient descent. These representations, though intuitive, often resulted in
large memory usage when capturing fine detail. So, the grid resolution was
reduced, and intermediate results were passed to an MLP network to bridge the
gap between the grid and finer details. This idea lead to the creation of the
multigrid parametric encoding \citep{Hadadan2021NeuralRadiosity} and the hash
grid encoding \citep{Muller2022InstantEncoding}. The encodings are very similar,
with the hash grid encoding being optimized for fast access on a graphics
processing unit. Empirically, these encodings have outperformed FFEs in many
graphics applications, such as 2D image regression and implicit surface
regression. Though parametric encodings have predominantly been applied to
graphics applications, the scientific machine learning community has started to
take notice and the hash grid encoding has been used to solve Burgers and the
Navier-Stokes equations \citep{Huang2023EfficientEncoding}.

The strong empirical performance of both types of encoding has sparked an
interest in understanding how they work. Some authors considered empirical
metrics such as gradient confusion and activation regions
\citep{Lazzari2023UnderstandingDynamics}, while others have investigated the
encodings through the lens of the neural tangent kernel
\citep{Tancik2020FourierDomains}. Though gradient confusion and activation
regions provide interesting insight, activation regions are restricted to ReLU
activation functions, and gradient confusion is heavily dependent on the
stochastic sampling in training. Therefore, we focus our analysis on the NTK. In
the infinite width limit, any common machine learning architecture can be
represented by a kernel regression problem using the NTK
\citep{Yang2019ANetworks}. Then, given a small enough learning rate, the
training dynamics are equivalent to solving a simple ordinary differential
equation \citep{Jacot2018NeuralNetworks}. This representation shows that
features along the largest eigenvalues will converge faster
\citep{Basri2020FrequencyDensity, Basri2019TheFrequencies,
Jacot2018NeuralNetworks, Tancik2020FourierDomains, Wang2020WhenPerspective}.
Though this kernel does not always exist in closed form, we can empirically
evaluate the finite width parallel using automatic differentiation
\citep{Novak2022FastKernel}. Analogous to a first order Taylor expansion of the
weights, the finite width kernel provides valuable NTK insights for the network
sizes used in practice. Previous analysis of the NTK for coordinate based MLPs
has been restricted to FFEs. We seek to include MPEs in this extensive
literature, and show that the eigenvalue spectrum can be used to explain MPE's
improved performance over FFEs and the baseline MLP network.

\section{Background and Notation}

In this section, we briefly explain the necessary background to understand our
key results. After an overview of our notation, we describe the structure and
mathematical formulation of the two most popular encodings. We then give a brief
overview of the NTK and how it relates to the spectral bias problem by connecting
the eigenvalues to the convergence on their corresponding eigenvector.

\textbf{Notation.} We consider a data set, $\mathbf{X} = \{\mathbf{x}_1,
\mathbf{x}_2, \ldots, \mathbf{x}_N\}$ where $\mathbf{x}_i \in \mathbb{R}^d \; ;
\; 0 < d \leq 3, i \in \{1, \ldots, N\}$. Samples are typically drawn from
Cartesian space or the pixel space of an image. The corresponding target
$\mathbf{Y} = \{\mathbf{y}_1, \mathbf{y}_2, \ldots, \mathbf{y}_N\}$ is similarly
low dimensional. The coordinate based MLP is then given by the function
$f_\theta(\mathbf{x}) : \mathbb{R}^{d_e} \rightarrow \mathbb{R}^{1, 2, \text{ or
} 3}$ parameterized by the learnable weights $\theta$, where $d_e$ is the
dimension of the embedding space. Matrices are denoted by bold capital letters,
while vectors are bold lower case letters. Encodings are denoted by
$\gamma(\mathbf{x}) : \mathbb{R}^d \rightarrow \mathbb{R}^{d_e}$, and the
overall network is given by $f_\theta \circ \gamma$.

\textbf{Axis-Aligned Logarithmic Fourier Feature Encoding.} Axis-aligned logarithmic Fourier feature encodings are implemented by passing the original network input through a series of sines and cosines, mimicking a Fourier transform. More frequencies are added by increasing a hyperparameter $L$. Making this a learnable parameter has not been shown to work well in practice \citep{Tancik2020FourierDomains}. $L$ needs to be balanced with the frequency content of the function, or else the result will begin to alias and decrease in quality \citep{Tancik2020FourierDomains}. Written out, the encoding appears as

 \begin{equation} \gamma_F(\mathbf{x}) = \left[\sin(2^0 \mathbf{x}), \cos(2^0\mathbf{x}),\sin(2^1 \mathbf{x}), \cos(2^1\mathbf{x}), \ldots, \sin(2^{L-1} \mathbf{x}), \cos(2^{L-1} \mathbf{x})  \right]. \label{eqn:ffe}\end{equation}

$\mathbf{x} \in \mathbb{R}^d$ and $\gamma_F(\mathbf{x}) \in \mathbb{R}^{d_e} = 2dL$. The logarithmic step helps the encoding shift to different frequencies due to
the cyclic nature of sines and cosines. Slight variations on this encoding exist, such
as phase parameters and coefficients being drawn from a standard normal distribution
\citep{Tancik2020FourierDomains}. 

\textbf{Axis-Aligned Multigrid Encoding.} Temporarily taking on the notation of
Euclidean coordinates, $x$ and $y$, we consider the 2D bilinear interpolation
function defined by,

\begin{equation} \tilde{g}(x, y) = \frac{1}{\Delta x \Delta y} 
\begin{bmatrix}
    x^{(n)}_2 - x & x - x^{(n)}_1 
\end{bmatrix} 
\begin{bmatrix} 
    g(x^{(n)}_1, y^{(n)}_1) & g(x^{(n)}_1, y^{(n)}_2) \\ 
    g(x^{(n)}_2, y^{(n)}_1) & g(x^{(n)}_2, y^{(n)}_2) 
\end{bmatrix} 
\begin{bmatrix}
    y^{(n)}_2 - y \\
    y - y^{(n)}_1
\end{bmatrix}.\end{equation}

Similar functions for linear and trilinear interpolation exist for 1D and 3D applications. On an evenly spaced grid, $\Delta x$ and $\Delta y$ represent the width and height of a grid cell. The coordinates $\left[x^{(n)}_1, y^{(n)}_1\right], \left[x^{(n)}_1, y^{(n)}_2\right],\left[x^{(n)}_2, y^{(n)}_2\right], \text{and} \left[x^{(n)}_2, y^{(n)}_1\right]$ represent the four corners of the $n^{th}$ cell, starting in the top left corner and moving clockwise. $g(x, y)$ is then any arbitrary function that we are interpolating. 

We define a series of grids of decreasing resolution, either linearly or logarithmically, from which we interpolate the input at each coordinate. Each point in the grid contains one or more learnable scalar parameters, typically initialized by $\mathcal{N}(0, 0.01)$ \citep{Hadadan2021NeuralRadiosity}. The interpolation at each resolution is then concatenated \citep{Muller2022InstantEncoding} along with the original input before being passed to the network. The encoding is shown graphically in Figure \ref{fig:multigrid-diagram} and is given mathematically by

\begin{equation} \gamma_{M,\phi}(\mathbf{x}) = \tilde{g}_\phi^{(0, 0)}(\mathbf{x})\oplus \ldots\oplus \tilde{g}_\phi^{(0, k)}(\mathbf{x})\oplus \ldots\oplus \tilde{g}_\phi^{(L, k)}(\mathbf{x})\oplus \mathbf{x}. \label{eqn:pe}\end{equation}

$k$ is the number of the learnable, scalar values at a given grid point, and a
typical value for $k$ is in the single digits. Similarly, $L$ is the index of
the grid. In practice, only a handful of layers are needed for good results.
Similar encodings exist, such as the sparse multigrid encoding
\citep{Hadadan2021NeuralRadiosity} and the hash grid encoding
\citep{Muller2022InstantEncoding}. Both variations work similarly and decrease
the memory requirements of the encoding without changing their theoretical
properties.

\begin{figure}[t]
    \centering
    \includegraphics[width=0.70\linewidth]{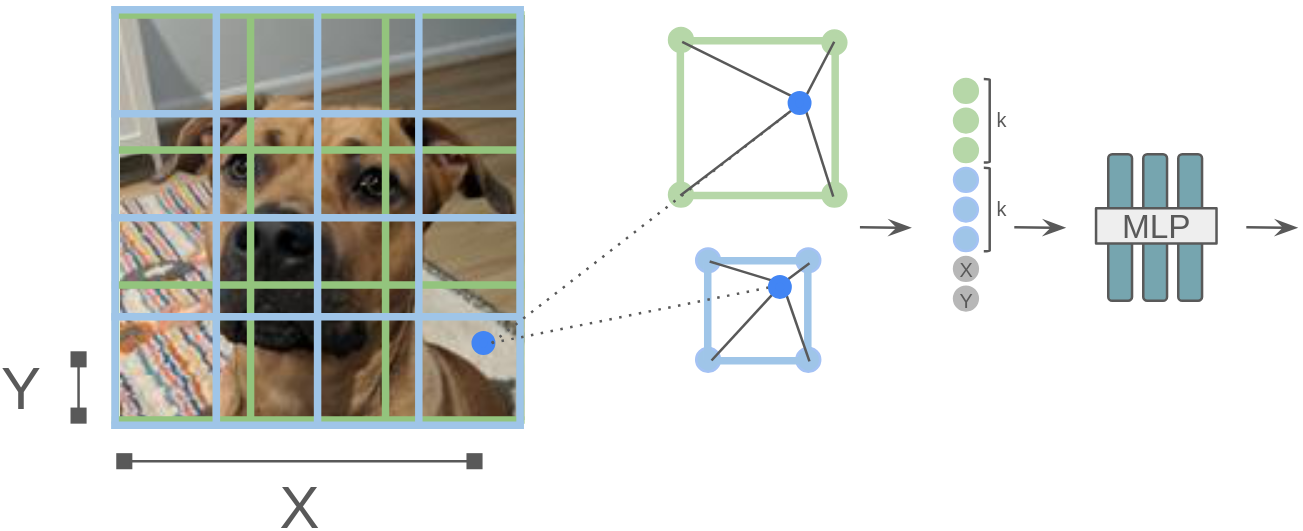}
    \caption{The above figure shows an example of the multigrid parametric
    encoding (MPE). The sample location (blue dot) is mapped to the surrounding
    grid cells (blue and green squares). The grid contains $k$ learnable scalars
    at each intersection point. Bilinear interpolation is performed on these
    learnable parameters independently. All learnable parameters are then
    concatenated with the origin x and y coordinate before being
    passed to the network.}
    \label{fig:multigrid-diagram}
\end{figure}

\textbf{Neural Tangent Kernel.} Recent work has made great strides in the
explainability of neural networks using what is known as the neural tangent
kernel. \citet{Jacot2018NeuralNetworks} have shown that as
the width of the network tends towards infinity and the learning rate of
stochastic gradient descent goes to zero, the training can be described as
kernel regression with the NTK. This result can be thought of intuitively as
the Taylor expansion about the optimal weights. The NTK then represents the first
order derivative in the expansion. This kernel is given by

\begin{equation} \mathbf{K}_{NTK}(\mathbf{x}_i, \mathbf{x}_j) = \mathbb{E}_{\theta \sim \mathcal{N}} \Big\langle \frac{\partial f_\theta(\mathbf{x}_i)}{\partial \theta}, \frac{\partial f_\theta(\mathbf{x}_j)}{\partial \theta} \Big\rangle. \end{equation}

The kernel is formed over all pairwise comparisons within the training data set.
This produces a positive semi-definite Gram matrix. With suitable initialization
and infinite width, this kernel becomes deterministic and remains constant
during training \citep{Jacot2018NeuralNetworks}. Assuming $K_{NTK}$ is
invertible, the network's predictions on a new set of data at training step $t$
is given by \citep{Wolberg2006KernelRegression}

\begin{equation} f_\theta(\mathbf{X}_{test}, t) \approx \mathbf{K}_{test}\mathbf{K}_{NTK}^{-1}(\mathbf{I} - e^{-\mathbf{K}_{NTK}t})\mathbf{Y}. \end{equation} 

$\mathbf{K}_{test}$ is given by the neural tangent kernel of the pairwise inner products of the training and test data set, and $\mathbf{Y}$ is the concatenated ground truth training data set as described in the notation section. Though originally formulated in the infinite width setting, the finite width evaluation of the NTK has been a reliable performance predictor of small and medium sized neural networks \citep{Tancik2020FourierDomains}.

\textbf{Spectral Bias.} The NTK allows a more classical understanding of the
network. $\mathbf{K}_{NTK}$ is positive semi-definite. Therefore, it can be
decomposed into an orthogonal matrix and a diagonal matrix containing the
eigenvalues $\lambda$ such that $\mathbf{K}_{NTK} = \mathbf{Q}^T
\mathbf{\Lambda} \mathbf{Q}$. Positive semi-definiteness means
that all $\lambda$'s are greater than or equal to $0$. We compare the
predictions of the training data by evaluating $f_\theta(\mathbf{X}, t) -
\mathbf{Y}$. Multiplying the NTK and its inverse produces the identity matrix,
giving 

\begin{equation} f_\theta(\mathbf{X}, t) - \mathbf{Y} = (\mathbf{I} - e^{-\mathbf{K}_{NTK}t})\mathbf{Y} - \mathbf{Y}.\end{equation}

Using the fact that $e^{-\mathbf{K}_{NTK}t} = \mathbf{Q}^Te^{-\mathbf{\Lambda} t}\mathbf{Q}$, our training loss simplifies to

\begin{equation} 
|\mathbf{Q}(f_\theta(\mathbf{X}, t) - \mathbf{Y})| = |-e^{-\mathbf{\Lambda}t}\mathbf{Q}\mathbf{Y}|. 
    \label{eqn:spectral-bias-eigenvalues}
\end{equation}

From Equation \ref{eqn:spectral-bias-eigenvalues} we see that the loss has a
dependence on the eigenvalues and eigenvectors of the NTK and that larger eigenvalues will
decrease the training error faster along that dimension. 
\citet{Basri2019TheFrequencies} found that eigenvalues corresponding to higher frequency relationships in the dataset are lower for coordinate-based MLPs. As a result, networks will fit lower frequency, smooth, relationships much faster than they will the higher frequencies. By comparing the
eigenvalue spectrum produced by the composition of encodings and the NTK, we can
evaluate an encodings' bias towards higher frequency features, which corresponds
to an increased number of large eigenvalues.

\section{Neural Tangent Kernels for Multigrid Parametric Encodings}

We explore the structure of the neural tangent kernel for different encodings
and prove that the MPE raises the eigenvalue spectrum of the kernel over the
baseline network. In special cases, the NTK can be written out in closed form
for its infinite width limit. We, however, focus on the finite width kernel to
evaluate training. By extending previous work on the NTK of MLP networks
\citep{Jacot2018NeuralNetworks, Yang2019ANetworks}, we
derive the kernel for MPEs and FFEs. We start by analyzing a single layer MLP
network, extend it to include encodings, and discuss how it is
easily extended for MLPs with multiple layers. 

Consider a single layer network without encoding and width $n$. Taking on the
parameterization scheme in \citet{Jacot2018NeuralNetworks}, we
scale the activation by the dimension $n$ and introduce a scalar, $\beta$, to
scale the bias. $\beta$ is typically set to 0.1. Given weights,
$\mathbf{W}^{(1)} \in \mathbb{R}^{n\times d}$ and $\mathbf{W}^{(2)} \in
\mathbb{R}^{1 \times n}$, and biases, $\mathbf{b} \in \mathbb{R}^{n\times 1}$,
we consider the function

\begin{equation}
    f_\theta(\mathbf{x}) = \mathbf{W}^{(2)}\frac{1}{\sqrt{n}}\phi(\mathbf{W}^{(1)}\mathbf{x} + \beta\mathbf{b}).
\end{equation}

Viewing $\theta$ as the concatenation of all learnable parameters, such that
$\theta = \big[\mathbf{W}^{(2)} \; \mathbf{W}^{(1)} \; \mathbf{b} \big]$, we
take the derivative with respect to each parameter matrix. The $i, j$ element of
$K_{NTK}$ is then given by the inner product between $\mathbf{x}_i$ and $\mathbf{x}_j$ in the data
set. Therefore, 

%%\begin{equation}
%%    \frac{\partial f_\theta}{\partial \theta} = 
%%    \begin{bmatrix}
%%        \frac{1}{\sqrt{n}}\phi(\mathbf{W}^{(1)}\mathbf{x} + \beta\mathbf{b}) & 
%%        \mathbf{W}^{(2)}\frac{1}{\sqrt{n}}\mathbf{x}\phi'(\mathbf{W}^{(1)}\mathbf{x} + \beta\mathbf{b}) &
%%        \mathbf{W}^{(2)}\frac{1}{\sqrt{n}}\beta\phi'(\mathbf{W}^{(1)}\mathbf{x} + \beta\mathbf{b})
%%    \end{bmatrix}.
%%\end{equation}

\begin{multline}
    \mathbf{K}^{(i,j)}_{NTK} = \mathbb{E}_{\theta \sim \mathcal{N}}\langle \frac{\partial f_\theta(\mathbf{x}_i)}{\partial \theta},
    \frac{\partial f_\theta(\mathbf{x}_j)}{\partial \theta}\rangle = \mathbb{E}_{\theta \sim \mathcal{N}}\big[\frac{1}{n}\langle
    \phi(\mathbf{W}^{(1)}\mathbf{x}_i+\beta\mathbf{b}),\phi(\mathbf{W}^{(1)}\mathbf{x}_j+\beta\mathbf{b})
    \rangle \\
    + \frac{1}{n}\langle \mathbf{W}^{(2)}\mathbf{x}_i\frac{\partial \phi}{\partial \theta}(\mathbf{W}^{(1)}\mathbf{x}_i
    + \beta\mathbf{b}), \mathbf{W}^{(2)}\mathbf{x}_j\frac{\partial \phi}{\partial \theta}(\mathbf{W}^{(1)}\mathbf{x}_j +
    \beta\mathbf{b})\rangle \\
    + \frac{\beta^2}{n}\langle \mathbf{W}^{(2)}\frac{\partial \phi}{\partial \theta}(\mathbf{W}^{(1)}\mathbf{x}_i
    + \beta\mathbf{b}), \mathbf{W}^{(2)}\frac{\partial \phi}{\partial \theta}(\mathbf{W}^{(1)}\mathbf{x}_j +
    \beta\mathbf{b})\rangle\big].
    \label{eqn:ntk-single}
\end{multline}

We now include the multigrid parametric encoding. As gradients are computed
element wise across the training dataset, we only need to consider a single grid
cell, layer, and learnable scalar at a time. Using the notation from section 3.3, define
$\Delta x_2 = x^{(n)}_2 - x, \Delta x_1 = x - x_1^{(n)}, \Delta y_2 = y_2^{(n)}
- y$, and $\Delta y_1 = y - y_1^{(n)}$. We drop the $n$ superscript to consider
a single cell, with learnable weights $w_{11}, w_{12}, w_{21},$ and $w_{22}$.
Writing out the matrices and multiplying through, we get

\begin{equation}
    \begin{bmatrix}
        \Delta x_2 & \Delta x_1   
    \end{bmatrix}
    \begin{bmatrix}
        w_{11} & w_{12} \\ w_{21} & w_{22}
    \end{bmatrix}
    \begin{bmatrix}
        \Delta y_2 \\ \Delta y_1
    \end{bmatrix} = 
        w_{11}\Delta y_2\Delta x_2 + w_{12}\Delta y_1\Delta x_2 + w_{21}\Delta y_2\Delta x_1 + w_{22}\Delta y_1\Delta x_1.
\end{equation}

It should be apparent that the gradient with respect to the weights in the grid cell is

\begin{equation}
    \nabla_\mathbf{W} \tilde{g}(x, y) = \tilde{g}'(x, y) = 
    \frac{1}{\Delta x\Delta y}
    \begin{bmatrix}
        \Delta x_2 \\ \Delta x_1    
    \end{bmatrix}
    \begin{bmatrix}
        \Delta y_2 & \Delta y_1
    \end{bmatrix}.
\end{equation}

The new function is the composition of $f_\theta$ with the encoding,
$\gamma_\theta$. As before, we concatenate all learnable parameters and compute
the $i^{th}, j^{th}$ index of $K_{NTK}$. The gradients of the grid parameters are independent due to the concatenation, so we first compute the contribution of a single grid parameter. The full kernel is then the sum of each grid cell contribution. Let $\tilde{g}(\mathbf{x}_i)$ be the
output of the bilinear interpolation on a single grid cell, then the kernel
element is

%%\begin{equation}
    %%f_\theta(\gamma_\theta(\mathbf{x})) = \mathbf{W}^{(2)}\phi(\mathbf{W}^{(1)}
    %%(\frac{1}{\Delta x \Delta y}
    %%\begin{bmatrix}
        %%\Delta x_2 & \Delta x_1   
    %%\end{bmatrix}
    %%\begin{bmatrix}
        %%w_{11} & w_{12} \\ w_{21} & w_{22}
    %%\end{bmatrix}
    %%\begin{bmatrix}
        %%\Delta y_2 \\ \Delta y_1
    %%\end{bmatrix}) + \mathbf{b}). 
%%\end{equation}

\begin{multline}
    \mathbf{K}^{(i,j)}_{NTK_{MPE}} = \mathbb{E}_{\theta \sim \mathcal{N}}\langle \frac{\partial f_\theta(\gamma_\theta(\mathbf{x}_i))}{\partial \theta},
    \frac{\partial f_\theta(\gamma_\theta(\mathbf{x}_j))}{\partial \theta}\rangle = \\ \mathbb{E}_{\theta \sim \mathcal{N}}\big[\frac{1}{n}\langle
    \phi(\mathbf{W}^{(1)}\tilde{g}(\mathbf{x}_i)+\beta\mathbf{b}),\phi(\mathbf{W}^{(1)}\tilde{g}(\mathbf{x}_j)+\beta\mathbf{b})
    \rangle \\
    + \frac{1}{n}\langle \mathbf{W}^{(2)}\tilde{g}(\mathbf{x}_i)\frac{\partial \phi}{\partial \theta}(\mathbf{W}^{(1)}\tilde{g}(\mathbf{x}_i)
    + \beta\mathbf{b}), \mathbf{W}^{(2)}\tilde{g}(\mathbf{x}_j)\frac{\partial \phi}{\partial \theta}(\mathbf{W}^{(1)}\tilde{g}(\mathbf{x}_j) +
    \beta\mathbf{b})\rangle \\
    + \frac{\beta^2}{n}\langle \mathbf{W}^{(2)}\frac{\partial \phi}{\partial \theta}(\mathbf{W}^{(1)}\tilde{g}(\mathbf{x}_i)
    + \beta\mathbf{b}), \mathbf{W}^{(2)}\frac{\partial \phi}{\partial \theta}(\mathbf{W}^{(1)}\tilde{g}(\mathbf{x}_j) +
    \beta\mathbf{b})\rangle \\
    + \frac{1}{n}\langle \mathbf{W}^{(2)}\mathbf{W}^{(1)}\tilde{g}'(\mathbf{x}_i)\frac{\partial \phi}{\partial \theta}(\mathbf{W}^{(1)}\tilde{g}(\mathbf{x}_i)
    + \beta\mathbf{b}), \mathbf{W}^{(2)}\mathbf{W}^{(1)}\tilde{g}'(\mathbf{x}_j)\frac{\partial \phi}{\partial \theta}(\mathbf{W}^{(1)}\tilde{g}(\mathbf{x}_j) +
    \beta\mathbf{b})\rangle\big].
    \label{eqn:ntk-mpe-single}
\end{multline}

The first three terms correspond to the original MLP with the input in the
embedding space of the MPE instead of the original coordinate space. Let's
denote this with $\mathbf{K}^{i, j}_{MLP}$. The last term is a new
term induced by the parameters in the grid, which we'll label
$\mathbf{K}^{i,j}_{MPE}$. As layers along the $L$ dimension and trainable
parameters along the $k$ dimension are independent, their
kernel contribution is simply Equation \ref{eqn:ntk-mpe-single} repeated for
each layer and parameter plus the kernel evaluated on the original coordinates.
Expanding out across all elements $\mathbf{x}_i \in \mathbf{X}$, we get the
following kernel for the MPE:

\begin{equation}
    \mathbf{K}_{NTK_{MPE}} = \mathbf{K}_{NTK} + \sum_{l=1}^L\sum_{a=1}^k \mathbf{K}^{l, a}_{MLP} + 
    \sum_{l=1}^L\sum_{a=1}^k \mathbf{K}^{l, a}_{MPE}.
    \label{eqn:ntk-mpe-extended}
\end{equation}

The NTK induced by the composition of the MPE and a MLP is the sum of
the original NTK, Equation \ref{eqn:ntk-single}, without any encoding and the
NTK produced by each layer and learnable parameter, Equation
\ref{eqn:ntk-mpe-single}. Though we only discuss a single layer MLP to illustrate
the structure of the NTK, the extension to multiple layers is well established
and can be built recursively from the above definitions 
\citep{Jacot2018NeuralNetworks,Yang2019ANetworks}. Adding layers to the MLP
will simply add terms to Equation \ref{eqn:ntk-single} and Equation 
\ref{eqn:ntk-mpe-single}. Equation \ref{eqn:ntk-mpe-extended}, however,
still holds for deeper MLPs. We use this fact to prove that the eigenvalue
spectrum of the composed kernel is greater than that of the base MLP.

\begin{theorem} 
Given a dataset $\mathbf{X}$, with $n$ samples, the corresponding neural tangent
kernel for a MLP network, and the neural tangent kernel for the same dataset and
MLP composed with a MPE. The $i^{th}$ eigenvalue, sorted in descending
order, of each kernel follows $\lambda_i^{MLP} \leq \lambda_i^{MLP} + \lambda_n^{MPE} \leq \lambda_i^{MLP+MPE}; \forall i \in \{1, \ldots, n\}$. $\lambda_i^{MLP}$ are the eigenvalues
of the NTK for the network with no encoding. $\lambda_n^{MPE}$ is the smallest eigenvalue of the
matrix produced by the encoding layers shown in Equation \ref{eqn:ntk-mpe-extended}.
$\lambda_i^{MLP+MPE}$ are then the eigenvalues of the NTK for the encoded network.
\label{thm:mpe-bound}
\end{theorem}

\begin{proof} Let $\mathbf{K}_{MLP}$ by the neural tangent kernel for a MLP
network evaluated on a training dataset $\mathbf{X}^n$. Let $\mathbf{K}_{MPE}$
be the neural tangent kernel for the composed MLP and MPE evaluated on the same
dataset. From Equation \ref{eqn:ntk-mpe-extended} we see that $\mathbf{K}_{MPE}
= \mathbf{K}_{MLP} + \mathbf{K}^+$, where $\mathbf{K}^+$ is the kernel produced
by the sum over all learnable parameters in the grid. By construction, each
$\mathbf{K}$ are square, symmetric Gram matrices as they are constructed by the
inner product. Each $\mathbf{K}$ is, therefore, positive semidefinite, making
$\mathbf{K}_{MPE}$ the sum of positive semidefinite matrices. Let
$\lambda_i(\mathbf{K})$ denote the $i^{th}$ eigenvalue of the matrix
$\mathbf{K}$. It follows from Weyl's inequality
\citep{Weyl1912DasHohlraumstrahlung} that $\forall i \in \{1, \ldots, n\},
\lambda_i(\mathbf{K}_{MLP}) \leq \lambda_i(\mathbf{K}_{MLP}) +
\lambda_n(\mathbf{K}^+) \leq \lambda_i(\mathbf{K}_{MPE}).$ \end{proof}

We see that the MPE changes the kernel both through its embedding space and the
learnable parameters contained within the grid. By adding more layers in the
grid or more learnable parameters at the grid nodes, we can increase the
spectrum of the kernel at the cost of additional computation and memory usage.
It is possible that the minimum eigenvalue of the additional matrices is zero;
however, as will be demonstrated in the next section, this is far from true in
practice.

For comparison, the NTK for the FFE can be computed as
\begin{multline}
    \mathbf{K}^{(i,j)}_{NTK_{FFE}} = \mathbb{E}_{\theta \sim \mathcal{N}}\langle \frac{\partial f_\theta(\gamma_F(\mathbf{x}_i))}{\partial \theta},
    \frac{\partial f_\theta(\gamma_F(\mathbf{x}_j))}{\partial \theta}\rangle =  \\
    \mathbb{E}_{\theta \sim \mathcal{N}}\big[\frac{1}{n}\langle
    \phi(\mathbf{W}^{(1)}\gamma_F(\mathbf{x}_i)+\beta\mathbf{b}),\phi(\mathbf{W}^{(1)}\gamma_F(\mathbf{x}_j)+\beta\mathbf{b})
    \rangle \\
    + \frac{1}{n}\langle \mathbf{W}^{(2)}\gamma_F(\mathbf{x}_i)\frac{\partial \phi}{\partial \theta}(\mathbf{W}^{(1)}\gamma_F(\mathbf{x}_i)
    + \beta\mathbf{b}), \mathbf{W}^{(2)}\gamma_F(\mathbf{x}_j)\frac{\partial \phi}{\partial \theta}(\mathbf{W}^{(1)}\gamma_F(\mathbf{x}_j) +
    \beta\mathbf{b})\rangle \\
    + \frac{\beta^2}{n}\langle \mathbf{W}^{(2)}\frac{\partial \phi}{\partial \theta}(\mathbf{W}^{(1)}\gamma_F(\mathbf{x}_i)
    + \beta\mathbf{b}), \mathbf{W}^{(2)}\frac{\partial \phi}{\partial \theta}(\mathbf{W}^{(1)}\gamma_F(\mathbf{x}_j) +
    \beta\mathbf{b})\rangle\big].
    \label{eqn:ntk-ffe-expanded}
\end{multline}

Equation \ref{eqn:ntk-ffe-expanded} shows that the FFE improves the NTK solely
through its embedding space, and Theorem \ref{thm:mpe-bound} does not apply. The
MPE, in contrast, has two mechanism to influence the kernel, but which is
dominant? To isolate the improvements in the MPE, we compute the spectrum both
with and without the $\mathbf{K}_{MPE}$ term (Figure
\ref{fig:isolate_kplus}). Without the contributions of the learnable grid,
the MPE has little to no effect on the eigenvalues as compared to baseline. We
conclude that the MPE derives its performance from the learnable parameters and
not the higher dimensional embedding space, while the FFE's performance rests
solely on the embedding.

\begin{figure}[t]
    \centering
    \includegraphics[width=0.5\linewidth]{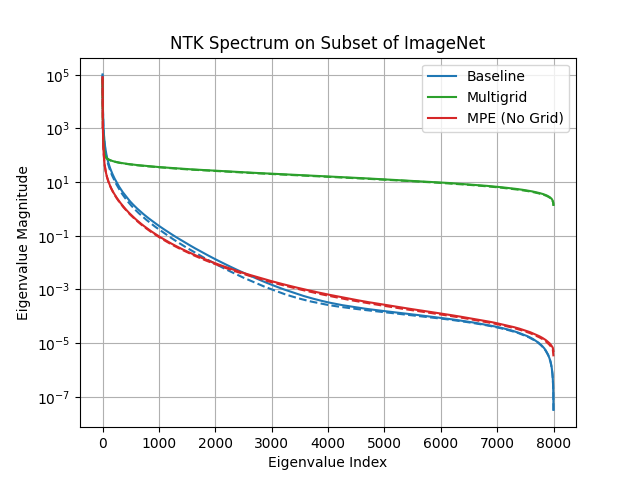}
    \caption{This plot isolates the improvements of the MPE to the learnable
    grid. The NTK is computed for a 2D image regression on random images from
    100 synonym sets in ImageNet (See the \textit{Experiments and Results}
    section for more details). The solid line and dashed line are the spectrum
    at the end and the middle of training, respectively. MPE (No Grid) is the
    NTK of the MPE without the contributions of $K_{MPE}$ and is purely for
    theoretical analysis. Without the grid, the spectrum is barely above the baseline. With
    the grid, the spectrum is 8 magnitudes higher.}
    \label{fig:isolate_kplus}
\end{figure}

\section{Experiments and Results}

We present an empirical analysis of the neural tangent kernel for multiple configuration of the FFE and MPE in the context of a 2D image regression and a 3D implicit surface regression problem. Additional results can be found in Appendix A. The image regression problem provides a controlled setting that is easily interpretable, allowing us to understand how the theoretical properties of the kernel translate to empirical results, considering convergence rates, prediction accuracy, and sensitivity to hyperparameters. 3D implicit surface regression demonstrates that the theory holds in higher dimensions as well, and provides a theoretical basis for previous results by \citet{Muller2022InstantEncoding}.

\textbf{Experimental Setup.} We evaluate the NTK in three settings. The first
setting explores the relationship between the eigenvalue spectrum and the
encoding hyperparameters on a single image. The second setting extends our
results to a larger class of images by evaluating the NTK on randomly sampled
images from 100 synonym sets from ImageNet \citep{Deng2009ImageNet:Database}.
The final setting demonstrates that the theory holds when learning a 3D implicit
surface from a mesh. The network size was held constant in each domain. For
image regression, 2 hidden layers with 512 neurons each were used, while the 3D
surface regression used 8 hidden layers with 256 neurons each. The ReLU
activation function was use in intermediate layers. For more details on 2D
regression, please see Appendix A and see Appendix E for 3D implicit surface
information. The baseline corresponds to the network with no encoding, while the
FFE and MPE hyperparameters were set as follows. The scaling experiment
hyperparameters (Table \ref{tab:scaling-params}) were hand selected, while the
ImageNet hyperparameters (Table \ref{tab:tuning-results}) and 3D implicit
surface hyperparameters (Table \ref{tab:implicit-3d-params}) were selected using
Optuna \citep{Akiba2019Optuna:Framework}.

\textbf{Evaluating the NTK and Regression.} We used automatic differentiation
and the fast finite width NTK calculation \citep{Novak2022FastKernel} to evaluate
the kernel during training. The eigenvalues of the kernel were computed at the
beginning, middle, and end of training. The eigenvalues are then sorted from
high to low and plotted on a logarithmic scale. To evaluate the regressed image,
we report both peak signal-to-noise ratio (PSNR) in dB and multi-scale structural similarity index measure (MS-SSIM) \cite{Wang2003Multi-scaleAssessment}. PSNR reports the max signal power as compared to the noise (average error); however, as PSNR is an MSE based metric, it may not capture fine detail, so we also include MS-SSIM which is more sensitive to the fine feature of the image. Higher scores in both metrics corresponds to higher quality outputs from the network as compared to ground truth.

\textbf{Scaling Analysis.} Figure \ref{fig:image-results} shows the regressed
image at the end of training. Qualitatively, we can see the spectral bias and how
the encodings mitigate it. The baseline is only able to learn a blurred image.
The low frequency FFE greatly improves the results but still lacks sharpness.
This fine detail is learned more easily by the higher frequency FFEs. Both MPEs
easily learn high frequency details without greatly increasing the network's
input size. This qualitative analysis is corroborated by the PSNR and MS-SSIM values in Table 
\ref{tab:scaling-params}.

\begin{table}[h]
    \centering
    \begin{tabular}{c|ccc|cccccc}
    & \multicolumn{3}{|c|}{ImageNet} & \multicolumn{6}{c}{Scaling} \\
    Metric & Multigrid & Fourier & Baseline & Low & Mid & High & Coarse & Fine & Baseline \\
    \hline
    PSNR $\uparrow$ & 45.28 & 33.47 & 29.94 & 25 & 35 & 36.5 & 34.5 & 41.5 & 20 \\
    m-SSIM $\uparrow$ & 0.79 & 0.39 & 0.32 & 0.20 & 0.23 & 0.40 & 0.37 & 0.73 & 0.08 \\
    \end{tabular}
    \caption{This table presents image quality metrics between the regressed image and ground truth across experiments. PSNR and MS-SSIM evaluate this similarity, with the latter being more sensitive to fine detail. As expected, we see that the MPE has the highest scores, while the FFE and MPE perform well over the baseline. See Figure \ref{fig:image-results} for a visual representation of these trends.}
    \label{tab:metrics}
\end{table}

\begin{figure}[t]
    \centering
    \includegraphics[width=0.7\linewidth]{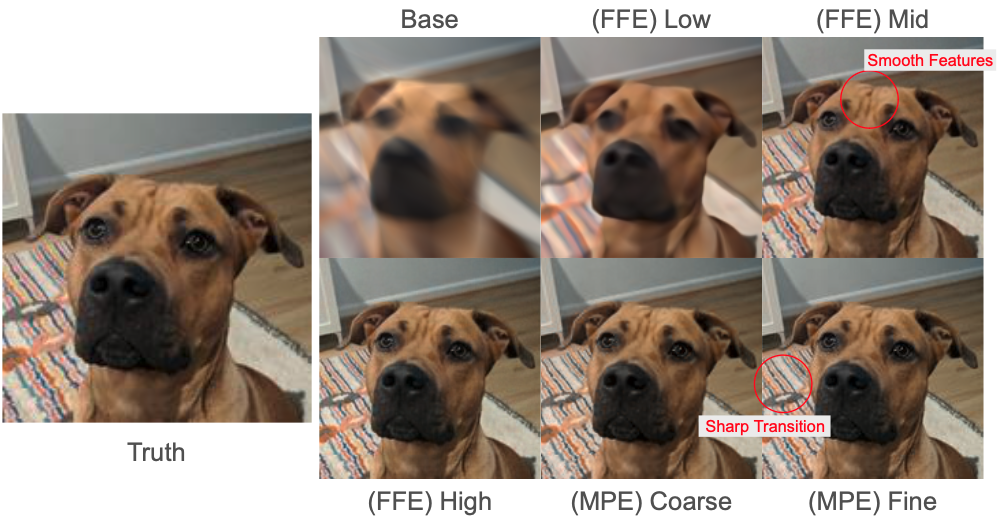}
    \caption{We compare performance of different encodings on image regression.
    We show the ground truth image (leftmost) along with a network with
    no encoding (top left), 3 configurations of the Fourier feature encoding
    (FFE), and two configurations of the multigrid parametric encoding (MPE)
    (see Table \ref{tab:scaling-params}). No encoding produces a blurred image,
    but as we add encodings, the finer details start to be resolved. An increase
    in detail is seen in the FFEs, while the MPEs perform well with both the coarse
    and the fine grid.}
    \label{fig:image-results}
\end{figure}

Figure \ref{fig:eigenvalue-2d-regression} plots the corresponding eigenvalue
spectrum for each of these regression problems shown in Figure
\ref{fig:image-results}. The eigenvalue spectrum at the middle and end of
training are plotted as the dashed and solid lines, respectively. The baseline
encoding has the worst performance, with the eigenvalues quickly dipping below
$1e-6$. The left most plot shows the comparison against all three FFEs. As
expected, increasing frequency corresponds to an increase in the eigenvalues;
however, we see that the FFE saturates, and there is little improvement between
the mid and high frequencies. The middle figure plots the same comparison for
the MPEs. Again, as theory predicts, the higher performing fine MPE has a higher
eigenvalue spectrum than the coarse grid. The right figure then compares the two
MPEs to the high frequency FFE. The fine grid has the highest spectrum overall,
while the FFE and coarse MPE flip around the middle of the plot. This could
explain why the PSNR and MS-SSIM values for the two encodings are so similar.
The coarse MPE will quickly learn finer detail than the baseline, but the FFE
will be able to learn the very fine details at the smallest eigenvalues.
Appendix A shows the same results across encodings on additional images.

\begin{figure}[t]
    \centering
    \includegraphics[width=1.0\linewidth]{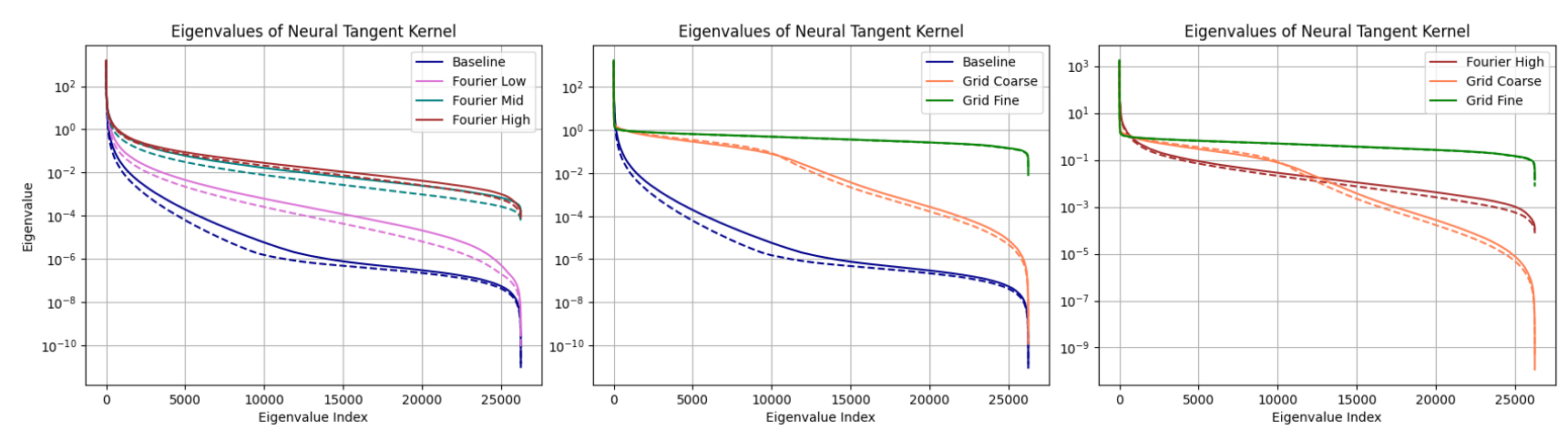}
    \caption{The NTK eigenvalue spectrum is compared for the cases found in
    Figure \ref{fig:image-results} and Table \ref{tab:scaling-params}. All
    encodings perform better than the baseline. The left plot shows the
    comparison of the baseline and the FFEs. The middle plot shows the same for
    the baseline and the MPEs. The right blot then compares the high frequency
    FFE to both MPE. The FFEs seem to saturate, while the fine grid MPE
    gives the best performance. These trends are backed by the PSNR
    values reported in the table. Interestingly, the coarse MPE crosses
    over the FFE, giving it strong performance early in training but allowing
    for higher PSNR values in the FFEs at the end of training.}
    \label{fig:eigenvalue-2d-regression}
\end{figure}

\textbf{ImageNet Analysis.} Plots of the average eigenvalue spectrum for the tuned
encodings on ImageNet are found in Figure \ref{fig:imagenet-spectra}. Again,
solid lines and dashed lines are the spectra at the end and middle of training,
respectively. As images could have different numbers of pixels, the spectra were
scaled to 8000 values. The mean was computed for each category. Tuning reflects
the use of encodings in practice, and we find that the MPE shows clear benefits
in the spectrum on a wide variety of images. The clear delineations of the
spectra also demonstrates the robustness of this evaluation on a wide class of
images and strong alignment between the performance of the network and the
spectrum of the NTK.

\textbf{OccupancyNet Analysis.}
Lastly, we plot the average eigenvalue spectrum for evaluations across three 3D meshes from the Stanford graphics dataset \cite{Turk1994ZipperedImages} in Figure \ref{fig:occupancynet-spectra}. Again, we see a clear benefit of the MPE, with the spectra raising above the other encodings. The minimum eigenvalues of the FFE and MPE are similar, but the overall spectrum for the MPE is higher. This corresponds to better training across all eigenvectors. Both encodings show a clear improvement over the baseline, as expected. Visualization of the 3D surface are in Appendix E. These results show that the theory easily extends from 2D to 3D, and that the NTK spectrum is a powerful tool in understanding the performance of encodings across domains.

\begin{figure}[t]
    \centering
    \begin{subfigure}{0.45\linewidth}
    \includegraphics[width=1.0\linewidth]{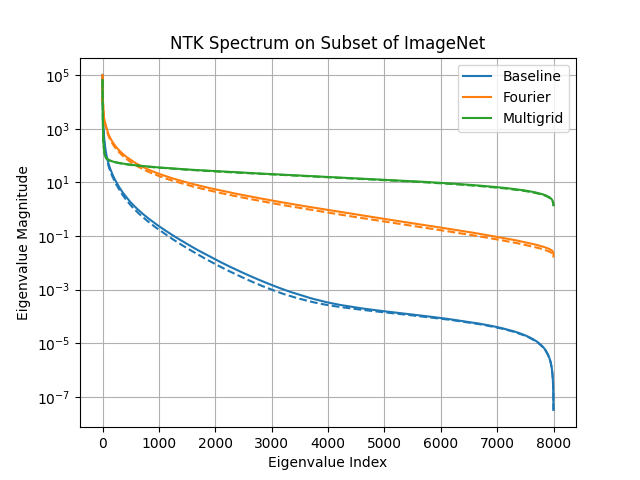}
    \caption{ImageNet}
    \label{fig:imagenet-spectra}
    \end{subfigure}
    \begin{subfigure}{0.45\linewidth}
    \includegraphics[width=1.0\linewidth]{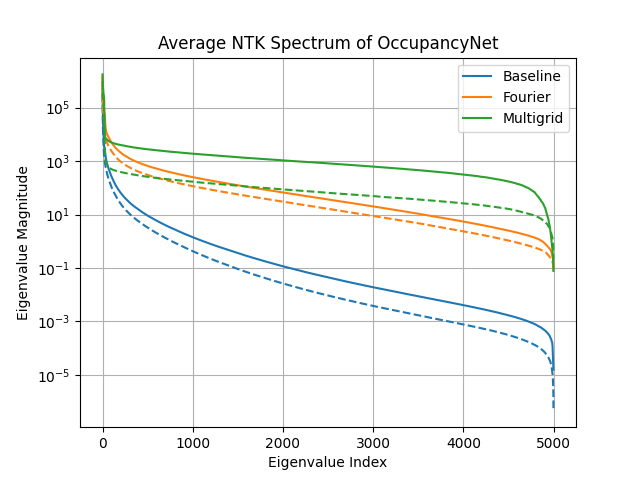}
    \caption{OccupancyNet}
    \label{fig:occupancynet-spectra}
    \end{subfigure}
    
    \caption{This figure compares the mean of the eigenvalue spectra of
    different encodings across randomly sampled images from 100 synonym sets in
    ImageNet and three 3D meshes on OccupancyNet. The dashed line shows the mean
    spectra at the midpoint of training. We find that tuned encodings have
    fairly regular performance: the MPE outperforms the FFE, which outperforms
    the baseline. This result shows that the spectrum is stable across different
    images and domains.}
    
\end{figure}

\section{Conclusion, Limitations, and Future Work}

We present an analysis of the multigrid parametric encoding compared to the
Fourier feature encoding and baseline coordinate-based MLP to understand the
MPE's strong empirical performance in many graphics applications. We follow
previous analyses on coordinate based MLPs for graphics and scientific machine
learning and leverage neural tangent kernel theory to connect the eigenvalue
spectrum with the learning of higher frequency information. We prove a lower
bound on the eigenvalues of the MPE and give strong empirical evidence for this
improvement being isolated to the learnable grid as opposed to the embedding
space. The MPE, therefore, mitigates spectral bias in a fundamentally different
way than the FFE. Our analysis, however, does not account for the influence of
different activation functions, and the space of possible problems is too large
to compute the NTK for each application. Evaluations on a 2D image regression
and 3D implicit surface problem demonstrated that the theory holds in multiple
domains, while providing easily interpreted visual results. Through metrics such
as PSNR and MS-SSIM, we showed strong alignment between the theory and empirical
results. Our work greatly improves the community's understanding of encodings
and could be leveraged in future works to better optimize them to specific
applications by adjusting the interpolation kernel or tuning other
hyperparameters. 

\newpage

\section*{Acknowledgments}

This project was supported in part by a grant from an NSF CAREER AWARD 1942230, ONR YIP award N00014-22-1-2271, ARO’s Early Career Program Award 310902-00001, Army Grant No. W911NF2120076, the NSF award CCF2212458, NSF Award No. 2229885 (NSF Institute for Trustworthy AI in Law and Society, TRAILS), a MURI grant 14262683, an award from meta 314593-00001, and an award from Capital One.

\bibliography{iclr2025_conference, references}

\begin{thebibliography}{39}
\providecommand{\natexlab}[1]{#1}
\providecommand{\url}[1]{\texttt{#1}}
\expandafter\ifx\csname urlstyle\endcsname\relax
  \providecommand{\doi}[1]{doi: #1}\else
  \providecommand{\doi}{doi: \begingroup \urlstyle{rm}\Url}\fi

\bibitem[Akiba et~al.(2019)Akiba, Sano, Yanase, Ohta, and Koyama]{Akiba2019Optuna:Framework}
Takuya Akiba, Shotaro Sano, Toshihiko Yanase, Takeru Ohta, and Masanori Koyama.
\newblock {Optuna: A Next-generation Hyperparameter Optimization Framework}.
\newblock \emph{Proceedings of the ACM SIGKDD International Conference on Knowledge Discovery and Data Mining}, pp.\  2623--2631, 7 2019.
\newblock \doi{10.1145/3292500.3330701}.
\newblock URL \url{https://arxiv.org/abs/1907.10902v1}.

\bibitem[Barron et~al.(2021)Barron, Mildenhall, Tancik, Hedman, Martin-Brualla, and Srinivasan]{Barron2021Mip-NeRF:Fields}
Jonathan~T. Barron, Ben Mildenhall, Matthew Tancik, Peter Hedman, Ricardo Martin-Brualla, and Pratul~P. Srinivasan.
\newblock {Mip-NeRF: A Multiscale Representation for Anti-Aliasing Neural Radiance Fields}.
\newblock \emph{Proceedings of the IEEE International Conference on Computer Vision}, pp.\  5835--5844, 3 2021.
\newblock ISSN 15505499.
\newblock \doi{10.1109/ICCV48922.2021.00580}.
\newblock URL \url{https://arxiv.org/abs/2103.13415v3}.

\bibitem[Basri et~al.(2019)Basri, Jacobs, Kasten, and Kritchman]{Basri2019TheFrequencies}
Ronen Basri, David Jacobs, Yoni Kasten, and Shira Kritchman.
\newblock {The Convergence Rate of Neural Networks for Learned Functions of Different Frequencies}.
\newblock \emph{Advances in Neural Information Processing Systems}, 32, 6 2019.
\newblock ISSN 10495258.
\newblock URL \url{https://arxiv.org/abs/1906.00425v3}.

\bibitem[Basri et~al.(2020)Basri, Galun, Geifman, Jacobs, Kasten, and Kritchman]{Basri2020FrequencyDensity}
Ronen Basri, Meirav Galun, Amnon Geifman, David Jacobs, Yoni Kasten, and Shira Kritchman.
\newblock {Frequency Bias in Neural Networks for Input of Non-Uniform Density}.
\newblock \emph{37th International Conference on Machine Learning, ICML 2020}, PartF168147-1:\penalty0 662--671, 3 2020.
\newblock URL \url{https://arxiv.org/abs/2003.04560v1}.

\bibitem[Chibane et~al.(2020)Chibane, Alldieck, and Pons-Moll]{Chibane2020ImplicitCompletion}
Julian Chibane, Thiemo Alldieck, and Gerard Pons-Moll.
\newblock {Implicit Functions in Feature Space for 3D Shape Reconstruction and Completion}.
\newblock \emph{Proceedings of the IEEE Computer Society Conference on Computer Vision and Pattern Recognition}, pp.\  6968--6979, 3 2020.
\newblock ISSN 10636919.
\newblock \doi{10.1109/CVPR42600.2020.00700}.
\newblock URL \url{https://arxiv.org/abs/2003.01456v2}.

\bibitem[Cuomo et~al.(2022)Cuomo, Di~Cola, Giampaolo, Rozza, Raissi, and Piccialli]{Cuomo2022ScientificNext}
Salvatore Cuomo, Vincenzo~Schiano Di~Cola, Fabio Giampaolo, Gianluigi Rozza, Maziar Raissi, and Francesco Piccialli.
\newblock {Scientific Machine Learning through Physics-Informed Neural Networks: Where we are and What's next}.
\newblock \emph{Journal of Scientific Computing}, 92, 1 2022.
\newblock ISSN 15737691.
\newblock \doi{10.1007/s10915-022-01939-z}.
\newblock URL \url{https://arxiv.org/abs/2201.05624v4}.

\bibitem[Deng et~al.(2009)Deng, Dong, Socher, Li, Li, and Fei-Fei]{Deng2009ImageNet:Database}
Jia Deng, Wei Dong, Richard Socher, Li~Jia Li, Kai Li, and Li~Fei-Fei.
\newblock {ImageNet: A Large-Scale Hierarchical Image Database}.
\newblock \emph{2009 IEEE Conference on Computer Vision and Pattern Recognition, CVPR 2009}, pp.\  248--255, 2009.
\newblock \doi{10.1109/CVPR.2009.5206848}.

\bibitem[Essakine et~al.(2024)Essakine, Cheng, Cheng, Zhang, Deng, Zhu, Sch{\"{o}}nlieb, and Aviles-Rivero]{Essakine2024WhereSurvey}
Amer Essakine, Yanqi Cheng, Chun-Wun Cheng, Lipei Zhang, Zhongying Deng, Lei Zhu, Carola-Bibiane Sch{\"{o}}nlieb, and Angelica~I Aviles-Rivero.
\newblock {Where Do We Stand with Implicit Neural Representations? A Technical and Performance Survey}.
\newblock 11 2024.
\newblock URL \url{https://arxiv.org/abs/2411.03688v1}.

\bibitem[Fridovich-Keil et~al.(2021)Fridovich-Keil, Yu, Tancik, Chen, Recht, and Kanazawa]{Fridovich-Keil2021Plenoxels:Networks}
Sara Fridovich-Keil, Alex Yu, Matthew Tancik, Qinhong Chen, Benjamin Recht, and Angjoo Kanazawa.
\newblock {Plenoxels: Radiance Fields without Neural Networks}.
\newblock \emph{Proceedings of the IEEE Computer Society Conference on Computer Vision and Pattern Recognition}, 2022-June:\penalty0 5491--5500, 12 2021.
\newblock ISSN 10636919.
\newblock \doi{10.1109/CVPR52688.2022.00542}.
\newblock URL \url{https://arxiv.org/abs/2112.05131v1}.

\bibitem[Hadadan et~al.(2021)Hadadan, Chen, and Zwicker]{Hadadan2021NeuralRadiosity}
Saeed Hadadan, Shuhong Chen, and Matthias Zwicker.
\newblock {Neural radiosity}.
\newblock \emph{ACM Transactions on Graphics}, 40\penalty0 (6), 12 2021.
\newblock ISSN 15577368.
\newblock \doi{10.1145/3478513.3480569}.
\newblock URL \url{https://dl.acm.org/doi/10.1145/3478513.3480569}.

\bibitem[Hanin \& Rolnick(2019)Hanin and Rolnick]{Hanin2019DeepPatterns}
Boris Hanin and David Rolnick.
\newblock {Deep ReLU Networks Have Surprisingly Few Activation Patterns}.
\newblock \emph{Advances in Neural Information Processing Systems}, 32, 6 2019.
\newblock ISSN 10495258.
\newblock URL \url{https://arxiv.org/abs/1906.00904v2}.

\bibitem[Huang \& Alkhalifah(2023)Huang and Alkhalifah]{Huang2023EfficientEncoding}
Xinquan Huang and Tariq Alkhalifah.
\newblock {Efficient physics-informed neural networks using hash encoding}.
\newblock \emph{Journal of Computational Physics}, 501, 2 2023.
\newblock \doi{10.1016/j.jcp.2024.112760}.
\newblock URL \url{http://arxiv.org/abs/2302.13397 http://dx.doi.org/10.1016/j.jcp.2024.112760}.

\bibitem[Jacot et~al.(2018)Jacot, Gabriel, and Hongler]{Jacot2018NeuralNetworks}
Arthur Jacot, Franck Gabriel, and Clément Hongler.
\newblock {Neural Tangent Kernel: Convergence and Generalization in Neural Networks}.
\newblock \emph{Advances in Neural Information Processing Systems}, 2018-December:\penalty0 8571--8580, 6 2018.
\newblock ISSN 10495258.
\newblock URL \url{https://arxiv.org/abs/1806.07572v4}.

\bibitem[Kiefer \& Wolfowitz(1952)Kiefer and Wolfowitz]{Kiefer1952StochasticFunction}
J~Kiefer and J~Wolfowitz.
\newblock {Stochastic Estimation of the Maximum of a Regression Function}.
\newblock \emph{The Annals of Mathematical Statistics}, 23\penalty0 (3):\penalty0 462--466, 3 1952.
\newblock ISSN 00034851.
\newblock URL \url{http://www.jstor.org/stable/2236690}.

\bibitem[Lazzari \& Liu(2023)Lazzari and Liu]{Lazzari2023UnderstandingDynamics}
John Lazzari and Xiuwen Liu.
\newblock {Understanding the Spectral Bias of Coordinate Based MLPs Via Training Dynamics}.
\newblock 1 2023.
\newblock URL \url{https://arxiv.org/abs/2301.05816v4}.

\bibitem[Li et~al.(2020)Li, Ma, and Zhang]{Li2020LearningNTK}
Yuanzhi Li, Tengyu Ma, and Hongyang~R. Zhang.
\newblock {Learning Over-Parametrized Two-Layer ReLU Neural Networks beyond NTK}.
\newblock \emph{Proceedings of Machine Learning Research}, 125:\penalty0 2613--2682, 7 2020.
\newblock ISSN 26403498.
\newblock URL \url{https://arxiv.org/abs/2007.04596v1}.

\bibitem[Lu et~al.(2019)Lu, Meng, Mao, and Karniadakis]{Lu2019DeepXDE:Equations}
Lu~Lu, Xuhui Meng, Zhiping Mao, and George~E. Karniadakis.
\newblock {DeepXDE: A deep learning library for solving differential equations}.
\newblock \emph{SIAM Review}, 63\penalty0 (1):\penalty0 208--228, 7 2019.
\newblock \doi{10.1137/19M1274067}.
\newblock URL \url{http://arxiv.org/abs/1907.04502 http://dx.doi.org/10.1137/19M1274067}.

\bibitem[Mescheder et~al.(2018)Mescheder, Oechsle, Niemeyer, Nowozin, and Geiger]{Mescheder2018OccupancySpace}
Lars Mescheder, Michael Oechsle, Michael Niemeyer, Sebastian Nowozin, and Andreas Geiger.
\newblock {Occupancy Networks: Learning 3D Reconstruction in Function Space}.
\newblock \emph{Proceedings of the IEEE Computer Society Conference on Computer Vision and Pattern Recognition}, 2019-June:\penalty0 4455--4465, 12 2018.
\newblock ISSN 10636919.
\newblock \doi{10.1109/CVPR.2019.00459}.
\newblock URL \url{https://arxiv.org/abs/1812.03828v2}.

\bibitem[Mildenhall et~al.(2020)Mildenhall, Srinivasan, Tancik, Barron, Ramamoorthi, and Ng]{Mildenhall2020NeRF:Synthesis}
Ben Mildenhall, Pratul~P. Srinivasan, Matthew Tancik, Jonathan~T. Barron, Ravi Ramamoorthi, and Ren Ng.
\newblock {NeRF: Representing Scenes as Neural Radiance Fields for View Synthesis}.
\newblock \emph{Lecture Notes in Computer Science (including subseries Lecture Notes in Artificial Intelligence and Lecture Notes in Bioinformatics)}, 12346 LNCS:\penalty0 405--421, 3 2020.
\newblock ISSN 16113349.
\newblock \doi{10.1007/978-3-030-58452-8{\_}24}.
\newblock URL \url{https://arxiv.org/abs/2003.08934v2}.

\bibitem[M{\"{u}}ller et~al.(2021)M{\"{u}}ller, Rousselle, Nov{\'{a}}k, and Keller]{Muller2021Real-timeTracing}
Thomas M{\"{u}}ller, Fabrice Rousselle, Jan Nov{\'{a}}k, and Alexander Keller.
\newblock {Real-time neural radiance caching for path tracing}.
\newblock \emph{ACM Transactions on Graphics (TOG)}, 40\penalty0 (4), 7 2021.
\newblock ISSN 15577368.
\newblock \doi{10.1145/3450626.3459812}.
\newblock URL \url{https://dl.acm.org/doi/10.1145/3450626.3459812}.

\bibitem[M{\"{u}}ller et~al.(2022)M{\"{u}}ller, Evans, Schied, and Keller]{Muller2022InstantEncoding}
Thomas M{\"{u}}ller, Alex Evans, Christoph Schied, and Alexander Keller.
\newblock {Instant neural graphics primitives with a multiresolution hash encoding}.
\newblock \emph{ACM Transactions on Graphics (TOG)}, 41\penalty0 (4):\penalty0 102, 7 2022.
\newblock ISSN 15577368.
\newblock \doi{10.1145/3528223.3530127}.
\newblock URL \url{https://dl.acm.org/doi/10.1145/3528223.3530127}.

\bibitem[Novak et~al.(2022)Novak, Sohl-Dickstein, and Schoenholz]{Novak2022FastKernel}
Roman Novak, Jascha Sohl-Dickstein, and Samuel~S. Schoenholz.
\newblock {Fast Finite Width Neural Tangent Kernel}.
\newblock \emph{Proceedings of Machine Learning Research}, 162:\penalty0 17018--17044, 6 2022.
\newblock ISSN 26403498.
\newblock URL \url{https://arxiv.org/abs/2206.08720v1}.

\bibitem[Paszke et~al.(2019)Paszke, Gross, Massa, Lerer, Bradbury, Chanan, Killeen, Lin, Gimelshein, Antiga, Desmaison, K{\"{o}}pf, Yang, DeVito, Raison, Tejani, Chilamkurthy, Steiner, Fang, Bai, and Chintala]{Paszke2019PyTorch:Library}
Adam Paszke, Sam Gross, Francisco Massa, Adam Lerer, James Bradbury, Gregory Chanan, Trevor Killeen, Zeming Lin, Natalia Gimelshein, Luca Antiga, Alban Desmaison, Andreas K{\"{o}}pf, Edward Yang, Zach DeVito, Martin Raison, Alykhan Tejani, Sasank Chilamkurthy, Benoit Steiner, Lu~Fang, Junjie Bai, and Soumith Chintala.
\newblock {PyTorch: An Imperative Style, High-Performance Deep Learning Library}.
\newblock \emph{Advances in Neural Information Processing Systems}, 32, 12 2019.
\newblock ISSN 10495258.
\newblock URL \url{https://arxiv.org/abs/1912.01703v1}.

\bibitem[Pumarola et~al.(2020)Pumarola, Corona, Pons-Moll, and Moreno-Noguer]{Pumarola2020D-NeRF:Scenes}
Albert Pumarola, Enric Corona, Gerard Pons-Moll, and Francesc Moreno-Noguer.
\newblock {D-NeRF: Neural Radiance Fields for Dynamic Scenes}.
\newblock \emph{Proceedings of the IEEE Computer Society Conference on Computer Vision and Pattern Recognition}, pp.\  10313--10322, 11 2020.
\newblock ISSN 10636919.
\newblock \doi{10.1109/CVPR46437.2021.01018}.
\newblock URL \url{https://arxiv.org/abs/2011.13961v1}.

\bibitem[Rahimi \& Recht(2008)Rahimi and Recht]{Rahimi2008RandomMachines}
Ali Rahimi and Benjamin Recht.
\newblock {Random features for large-scale kernel machines}.
\newblock In \emph{Advances in Neural Information Processing Systems 20 - Proceedings of the 2007 Conference}, 2008.
\newblock ISBN 160560352X.

\bibitem[Raissi et~al.(2017)Raissi, Perdikaris, and Karniadakis]{Raissi2017PhysicsEquations}
Maziar Raissi, Paris Perdikaris, and George~Em Karniadakis.
\newblock {Physics Informed Deep Learning (Part I): Data-driven Solutions of Nonlinear Partial Differential Equations}.
\newblock 11 2017.
\newblock URL \url{https://arxiv.org/abs/1711.10561v1}.

\bibitem[Sch{\"{o}}lkopf et~al.(1999)Sch{\"{o}}lkopf, Burges, and Smola]{10.5555/299094}
Bernhard Sch{\"{o}}lkopf, Christopher J~C Burges, and Alexander~J Smola (eds.).
\newblock \emph{{Advances in kernel methods: support vector learning}}.
\newblock MIT Press, Cambridge, MA, USA, 1999.
\newblock ISBN 0262194163.

\bibitem[Tancik et~al.(2020)Tancik, Srinivasan, Mildenhall, Fridovich-Keil, Raghavan, Singhal, Ramamoorthi, Barron, and Ng]{Tancik2020FourierDomains}
Matthew Tancik, Pratul~P. Srinivasan, Ben Mildenhall, Sara Fridovich-Keil, Nithin Raghavan, Utkarsh Singhal, Ravi Ramamoorthi, Jonathan~T. Barron, and Ren Ng.
\newblock {Fourier Features Let Networks Learn High Frequency Functions in Low Dimensional Domains}.
\newblock \emph{Advances in Neural Information Processing Systems}, 2020-Decem, 6 2020.
\newblock ISSN 10495258.
\newblock URL \url{https://arxiv.org/abs/2006.10739v1}.

\bibitem[Turk \& Levoy(1994)Turk and Levoy]{Turk1994ZipperedImages}
Greg Turk and Marc Levoy.
\newblock {Zippered polygon meshes from Range images}.
\newblock \emph{Proceedings of the 21st Annual Conference on Computer Graphics and Interactive Techniques, SIGGRAPH 1994}, pp.\  311--318, 7 1994.
\newblock \doi{10.1145/192161.192241/SUPPL{\_}FILE/P311-TURK.PS}.
\newblock URL \url{https://dl.acm.org/doi/10.1145/192161.192241}.

\bibitem[Vaswani et~al.(2017)Vaswani, Shazeer, Parmar, Uszkoreit, Jones, Gomez, Kaiser, and Polosukhin]{Vaswani2017AttentionNeed}
Ashish Vaswani, Noam Shazeer, Niki Parmar, Jakob Uszkoreit, Llion Jones, Aidan~N. Gomez, Lukasz Kaiser, and Illia Polosukhin.
\newblock {Attention is all you need}.
\newblock In \emph{Advances in Neural Information Processing Systems}, volume 2017-Decem, pp.\  5999--6009. Neural information processing systems foundation, 6 2017.
\newblock ISBN 1706.03762v7.
\newblock URL \url{https://arxiv.org/abs/1706.03762v7}.

\bibitem[Wang et~al.(2021{\natexlab{a}})Wang, Liu, Yang, and Tong]{Wang2021SplineFields}
Peng~Shuai Wang, Yang Liu, Yu~Qi Yang, and Xin Tong.
\newblock {Spline Positional Encoding for Learning 3D Implicit Signed Distance Fields}.
\newblock \emph{IJCAI International Joint Conference on Artificial Intelligence}, pp.\  1091--1097, 6 2021{\natexlab{a}}.
\newblock ISSN 10450823.
\newblock \doi{10.24963/ijcai.2021/151}.
\newblock URL \url{https://arxiv.org/abs/2106.01553v2}.

\bibitem[Wang et~al.(2020)Wang, Yu, and Perdikaris]{Wang2020WhenPerspective}
Sifan Wang, Xinling Yu, and Paris Perdikaris.
\newblock {When and why PINNs fail to train: A neural tangent kernel perspective}.
\newblock \emph{Journal of Computational Physics}, 449, 7 2020.
\newblock ISSN 10902716.
\newblock \doi{10.1016/j.jcp.2021.110768}.
\newblock URL \url{https://arxiv.org/abs/2007.14527v1}.

\bibitem[Wang et~al.(2021{\natexlab{b}})Wang, Wang, and Perdikaris]{Wang2021OnNetworks}
Sifan Wang, Hanwen Wang, and Paris Perdikaris.
\newblock {On the eigenvector bias of Fourier feature networks: From regression to solving multi-scale PDEs with physics-informed neural networks}.
\newblock \emph{Computer Methods in Applied Mechanics and Engineering}, 384:\penalty0 113938, 10 2021{\natexlab{b}}.
\newblock ISSN 0045-7825.
\newblock \doi{10.1016/J.CMA.2021.113938}.

\bibitem[Wang et~al.(2003)Wang, Simoncelli, and Bovik]{Wang2003Multi-scaleAssessment}
Zhou Wang, Eero~P. Simoncelli, and Alan~C. Bovik.
\newblock {Multi-scale structural similarity for image quality assessment}.
\newblock \emph{Conference Record of the Asilomar Conference on Signals, Systems and Computers}, 2:\penalty0 1398--1402, 2003.
\newblock ISSN 10586393.
\newblock \doi{10.1109/ACSSC.2003.1292216}.

\bibitem[Weyl(1912)]{Weyl1912DasHohlraumstrahlung}
Hermann Weyl.
\newblock {Das asymptotische Verteilungsgesetz der Eigenwerte linearer partieller Differentialgleichungen (mit einer Anwendung auf die Theorie der Hohlraumstrahlung)}.
\newblock \emph{Mathematische Annalen}, 71\penalty0 (4):\penalty0 441--479, 12 1912.
\newblock ISSN 00255831.
\newblock \doi{10.1007/BF01456804/METRICS}.
\newblock URL \url{https://link.springer.com/article/10.1007/BF01456804}.

\bibitem[Wolberg(2006)]{Wolberg2006KernelRegression}
John Wolberg.
\newblock {Kernel Regression}.
\newblock In \emph{Data Analysis Using the Method of Least Squares}, pp.\  203--238. Springer-Verlag, 2 2006.
\newblock \doi{10.1007/3-540-31720-1{\_}7}.

\bibitem[Yang \& Salman(2019)Yang and Salman]{Yang2019ANetworks}
Greg Yang and Hadi Salman.
\newblock {A Fine-Grained Spectral Perspective on Neural Networks}.
\newblock 7 2019.
\newblock URL \url{https://arxiv.org/abs/1907.10599v4}.

\bibitem[Yu et~al.(2020)Yu, Ye, Tancik, and Kanazawa]{Yu2020PixelNeRF:Images}
Alex Yu, Vickie Ye, Matthew Tancik, and Angjoo Kanazawa.
\newblock {pixelNeRF: Neural Radiance Fields from One or Few Images}.
\newblock \emph{Proceedings of the IEEE Computer Society Conference on Computer Vision and Pattern Recognition}, pp.\  4576--4585, 12 2020.
\newblock ISSN 10636919.
\newblock \doi{10.1109/CVPR46437.2021.00455}.
\newblock URL \url{https://arxiv.org/abs/2012.02190v3}.

\bibitem[Zhang et~al.(2020)Zhang, Riegler, Snavely, Tech, and Koltun]{Zhang2020NeRF++:Fields}
Kai Zhang, Gernot Riegler, Noah Snavely, Cornell Tech, and Vladlen Koltun.
\newblock {NeRF++: Analyzing and Improving Neural Radiance Fields}.
\newblock 10 2020.
\newblock URL \url{https://arxiv.org/abs/2010.07492v2}.

\end{thebibliography}
\bibliographystyle{iclr2025_conference}

\newpage

\appendix

\section{Training Details and Further Experiments}

\subsection{Training Details}

The PyTorch \citep{Paszke2019PyTorch:Library} library was used to implement all models and
analyses. Runs were completed on internal grid infrastructure using a single
NVIDIA RTX A5000 graphics card, 8 CPU cores, and 56 GB of system memory. The GPU
handles training and computation of the neural tangent kernel, while the CPU
handles the eigenvalue decomposition, as it was found to be faster. Each analysis
takes approximately 30 minutes to complete.

2D image regression trains the network on a single image. The pixel coordinates and red, green, and blue color values were normalized to between 0 and 1. There was no split between training and test data because, even with the complete dataset, the coordinate based MLP is not able to learn the high frequency information in the image. Training was done using stochastic gradient descent \cite{Kiefer1952StochasticFunction} with a mean squared error loss between the linear activation of the final layer and the ground truth pixel value.

To evaluate the NTK during training, slight variations are made to the network
architecture. NTK initialization \citep{Jacot2018NeuralNetworks} of the linear
layers was used in which the weights and biases are sampled from
$\mathcal{N}(0, \mathbf{I})$ and a layer, $l$, is given by
$\sigma(\frac{1}{\sqrt{n^{(l)}}}\mathbf{W}^{(l)}\mathbf{x}^{(l)} + \beta
\mathbf{b}^{(l)})$, where $n^{(l)}$ is the size of the input vector, $x^{(l)}$,
to that layer. Though this parameterization has the same approximation power
as the standard MLP, the additional factor greatly reduces the gradients
during backpropagation, requiring us to set the learning rate to 100 on some experiments. 

\subsection{Dog Image}

\begin{figure}[h]
    \centering
    \includegraphics[width=0.9\linewidth]{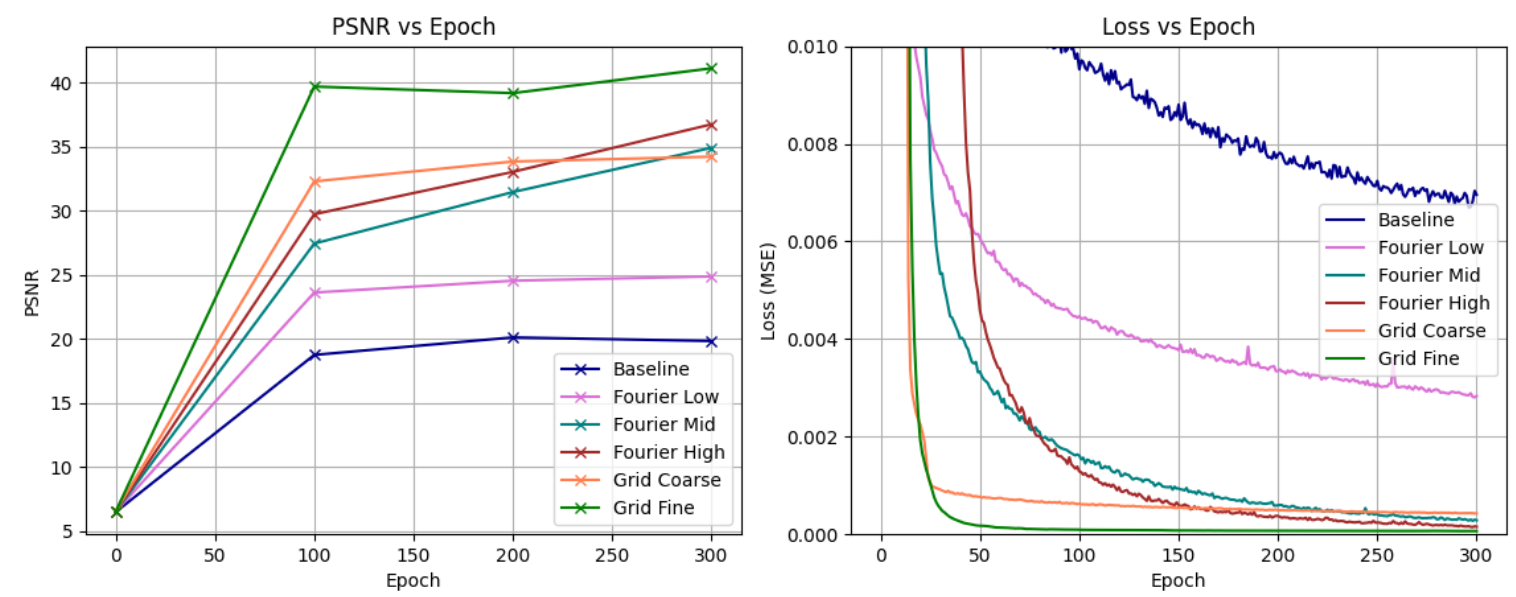}
    \caption{We report the peak signal-to-noise ratio (PSNR) and mean squared
    error (MSE) for the image regression problem shown in Figure
    \ref{fig:image-results}. We compare a baseline network with no encoding, 3
    Fourier feature encoding (FFE), and 2 multigrid parametric encoding (MPE).
    These results parallel what was seen in the regressed images. The MSE is
    highest and PSNR is lowest for the baseline encoding corresponding to the
    lowest quality image. As frequencies increase in the FFE, we see the MSE
    lower and the PSNR increase. The fine MPE outperforms even the highest
    frequency FFE by 5 dB PSNR, while the coarse MPE is on par with the higher frequency
    FFEs. We also note that the MPE lowers the error at a faster rate than the
    FFE, with loss dropping below 1e-3 125 epochs sooner.}
    \label{fig:psnr-loss}
\end{figure}

Figure \ref{fig:psnr-loss} compares the MSE loss and PSNR across encodings. As
expected, the baseline network has the highest training error and the lowest
PSNR. We find that these plots follow the qualitative results from Figure
\ref{fig:image-results}. The low frequency encoding halves MSE and increases the
PSNR by 5 dB. The mid and high frequency encodings show similar results in both
the regressed image and MSE and PSNR. The higher frequency encodings halve the
loss again and increase the PSNR by 10 dB as compared to the low frequency
encoding. The coarse multigrid encoding results in similar MSE and PSNR compared
to the mid and high frequency encodings; however, it trains much faster with the
loss dropping below 1e-3 125 epochs sooner than the mid and high frequency
encodings. The best performing encoding is then the fine multigrid encoding with
a PSNR another 5 dB higher than the coarse MPE, mid PPE, and high PPE. Loss is
also the lowest for the fine MPE. From these results, NTK theory would tell us
that the kernel's eigenvalues would follow the same trend.

\subsection{Windmill and Lake Images}

We include results on two other images for the 2D image regression problem
across encodings. We show that the same results from the main paper hold across
multiple data sets. We include both a natural lake scene and a
combination man made and natural scene featuring a windmill and flowerbed. Both
a qualitative look at the regressed images and MSE and PSNR show the same
trends. The base network is unable to learn fine details in the image. The three
FFEs show improving results with increasing frequency. The fine grid encoding
then has the best performance with the lowest training error, the largest PSNR,
and the highest eigenvalue spectrum.

\begin{figure}[h]
    \centering
    \includegraphics[width=0.80\linewidth]{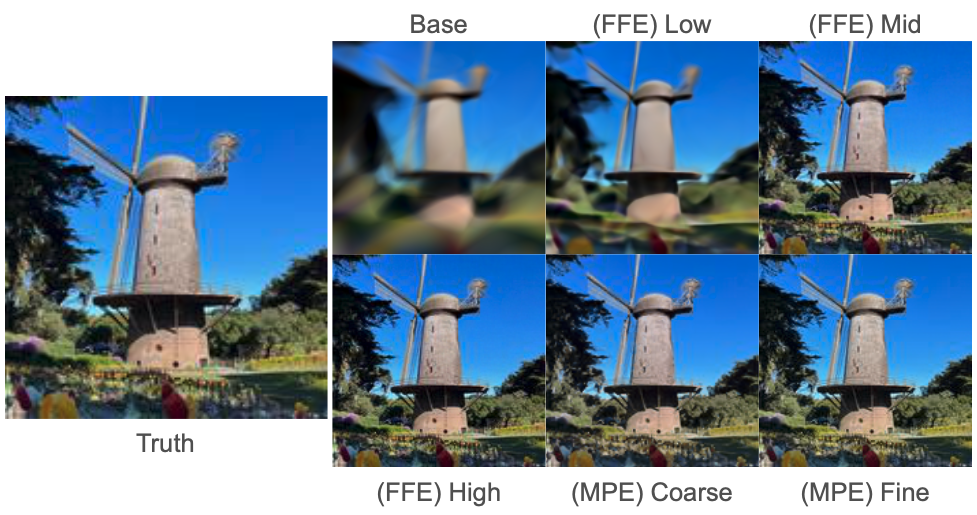}
    \caption{Results for the 2D image problem on an image of a windmill in a
    natural setting. See Section 5 for details on training and the parameters of
    each encoding. With no encoding (top left), the regressed image is blurry
    with no fine, high frequency details. The low frequency (top middle)
    encoding learns slightly more detail, but still results in a blocky image.
    The mid frequency encoding (top right), begins to show strong agreement with
    the ground truth image. The high frequency (bottom left), coarse grid
    (bottom middle), and fine grid (bottom right) encodings show even stronger
    agreement with the ground truth image, to the point where it is difficult to
    tell the difference with the human eye.}
    \label{fig:windmill-results}
\end{figure}

\begin{figure}[h]
    \centering
    \includegraphics[width=0.85\linewidth]{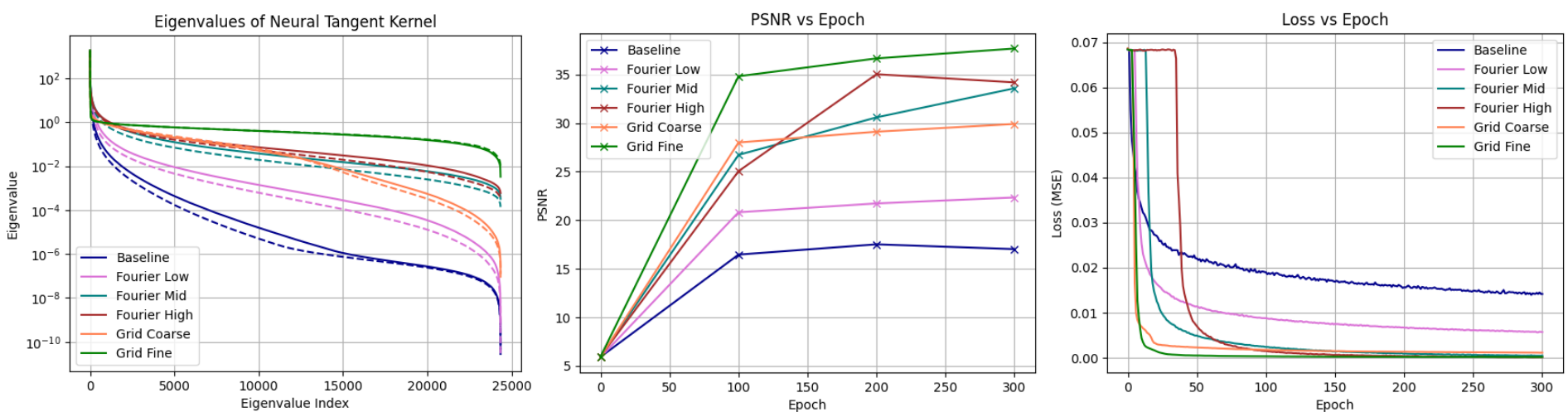}
    \caption{We plot the eigenvalue spectrum (left) at epoch 150 (dashed line)
    and epoch 300 (solid) line, the PSNR throughout training (middle), and the
    loss curves (right) to compare the performance of the encodings
    quantitatively on the windmill image. Similarly to the main paper, we find
    that the strongest performing encodings, highest PSNR and lowest training
    loss, has the highest magnitude eigenvalues spectrum. Lower eigenvalue
    spectra then correspond to lower PSNR and higher loss throughout
    training.}
    \label{fig:windmill-plots}
\end{figure}

\begin{figure}[h]
    \centering
    \includegraphics[width=0.80\linewidth]{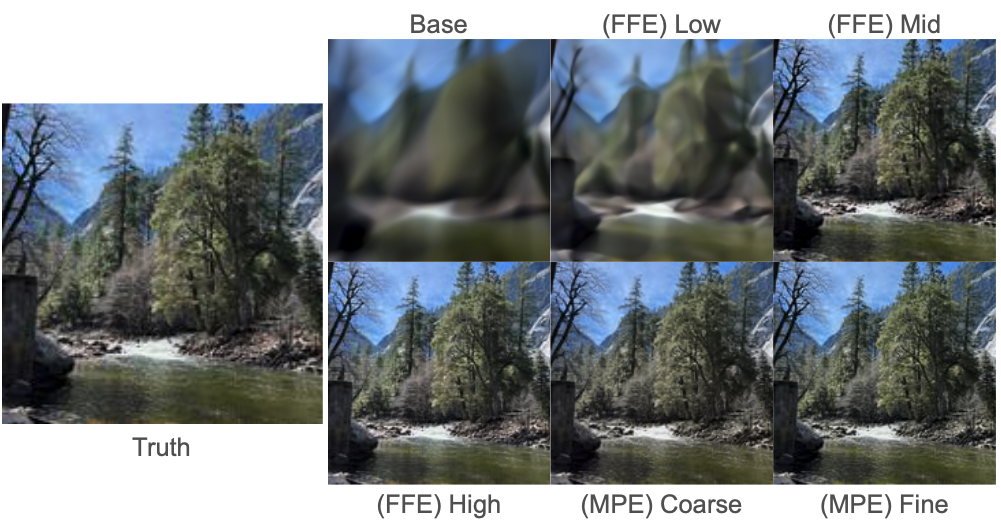}
    \caption{Results for the 2D image problem on an image of a lake in a
    natural setting. See Section 5 for details on training and the parameters of
    each encoding. With no encoding (top left), the regressed image is blurry
    with no fine, high frequency details. The low frequency (top middle)
    encoding learns slightly more detail, but still results in a blocky image.
    The mid-frequency encoding (top right), begins to show strong agreement with
    the ground truth image. The high frequency (bottom left), coarse grid
    (bottom middle), and fine grid (bottom right) encodings show even stronger
    agreement with the ground truth image, to the point where it is difficult to
    tell the difference with the human eye.}
    \label{fig:lake-results}
\end{figure}

\begin{figure}[h]
    \centering
    \includegraphics[width=0.85\linewidth]{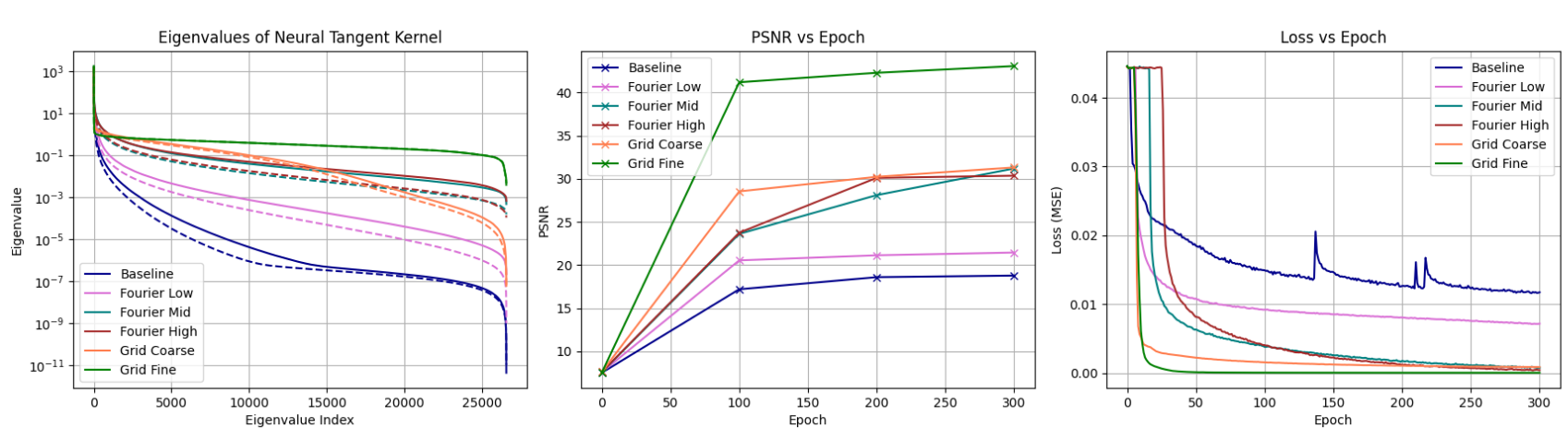}
    \caption{We plot the eigenvalue spectrum (left) at epoch 150 (dashed line)
    and epoch 300 (solid) line, the PSNR throughout training (middle), and the
    loss curves (right) to compare the performance of the encodings
    quantitatively on the lake image. Similarly to the main paper, we find
    that the strongest performing encodings, highest PSNR and lowest training
    loss, has the highest magnitude eigenvalues spectrum. Lower eigenvalue
    spectra then correspond to lower PSNR and higher loss throughout
    training.}
    \label{fig:lake-plots}
\end{figure}

\newpage

\section{Change of NTK During Training}

When plotting the eigenvalue spectrum, we only show results at epoch 150 and
epoch 300. In Figure \ref{fig:ntk-spectrum-start} we show the results for epoch
0 and epoch 300. As shown, the spectrum at the start of training is very
different from the spectrum at the end of training. We also note that
eigenvalues across encodings all seem to have similar values. NTK theory says
that the kernel is stable in the infinite width limit
\citep{Jacot2018NeuralNetworks}, and a width of 512 should be large enough to
see this effect. So why does that not hold in this case? Recent works by
\citet{Wang2021OnNetworks, Wang2020WhenPerspective} have shown that the
eigenvalue spectrum does, in fact, change during training in real world
scenarios. The reason given is that the initial initialization is too far from
the optimal parameter values, so large changes in weights are still necessary
for learning. This is most likely exacerbated by the fact our inputs are scaled
to the unit interval and not a unit normal distribution. This breaks the lazy
training assumption for the infinite width limit of the kernel; however, we are
focused on the finite width kernel and the training dynamics given different
encodings. Furthermore, the spectral bias analysis using the eigenvalues of the
NTK still holds. This can be seen by considering the Taylor expansion about the
optimal weights. Given slight perturbations, we expect the kernel to be
elatively stable in this region, and its eigenvalues allow us to conclude with
eigenvectors will be fit faster during training. This is reflected in the fact
that the kernel spectrum is stable after the start of training, as shown in the
main Figures.

\begin{figure}[ht]
    \centering
    \includegraphics[width=0.60\linewidth]{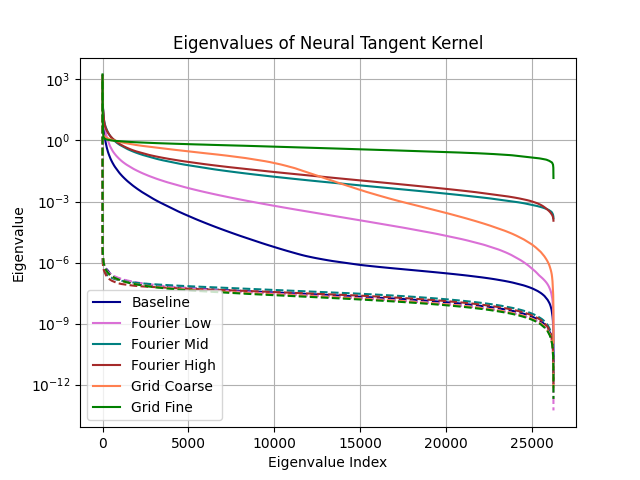}
    \caption{We plot the eigenvalues at the start of training (dotted line) and at epoch 300 (solid line). We 
    see that the spectrum is very low and close between encodings at the start of training, but raises and 
    separates by the end of training. Though this appears to break the lazy training observation
    of the infinite width kernel \citep{Jacot2018NeuralNetworks}, we are concerned with the finite
    width kernel and the spectral bias analysis still holds as long as we consider small regions about
    the weights at a time. As show in Figure \ref{fig:eigenvalue-2d-regression}, the kernel is
    relatively constant for the second half of training.}
    \label{fig:ntk-spectrum-start}
\end{figure}

\section{Visualizing Multigrid Parameters}

Though we have shown that the multigrid encoding is able to raise the eigenvalue spectrum
and expressiveness of a coordinate based MLP, it is also interesting to look at what
the parameters in the grid are learning. In Figure \ref{fig:multigrid-grayscale} we
plot the parameters as a gray scale image for both the coarse and fine encodings. Each
encoding contains two learnable scalars at each grid point. The $0^{th}$ parameter is shown
on the left while the $1^{st}$ is plotted on the right. We see that the grid, in fact, begins
to learn the image, while the backing MLP adjusts the color and helps interpolate. We also 
see small gaps in the fine encoding, most likely due to grid cells that were unused during training.
This shows that it's easy to waste space, especially with uniform sampling. In these cases,
the sparse and hash grid encodings could help reduce the memory footprint without changing
the analyses found in the main paper. 

\begin{figure}[ht]
    \centering
    \includegraphics[width=0.65\linewidth]{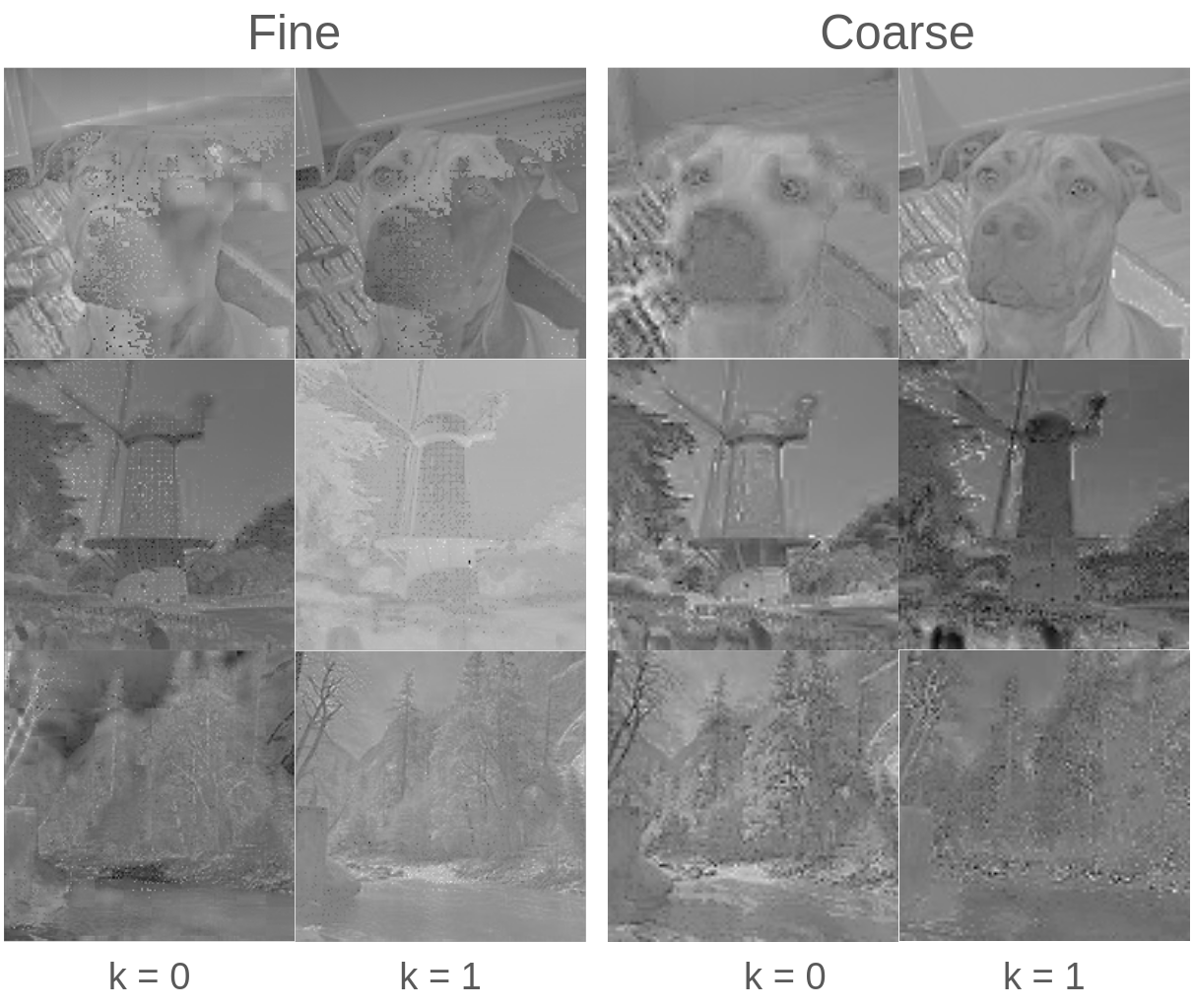}
    \caption{This figure shows the learned scalars at grid points in the MPE as a gray scale image.
    This allows us to get greater insight into what the encoding is doing. We plot results for both the
    fine grid encoding (left) and the coarse grid encoding (right). Each encoding has a single layer and
    two learnable parameters at each node. The $0^th$ parameter is shown in the left column and the $1^{st}$
    parameter is in the right column. We find that the encoding is learning representations similar
    to the regressed image. We also note that the fine grid encoding has small white patches in the middle
    most likely from where no pixel coordinate landed in the ground truth image.}
    \label{fig:multigrid-grayscale}
\end{figure}

If the multigrid encoding is learning the image, then does the MLP need to do
less work? Put another way, can we reduce the size of the MLP to save memory and
computation. To investigate this question, we plot the activation regions
\citep{Hanin2019DeepPatterns} in Figure \ref{fig:activation-regions}. Activation
regions encode which neurons in the MLP cause the ReLU activation function to
take on a value of 1. This produces a binary encoding that can be used to define
an activation pattern, $\mathcal{A}$, which is a vector in $\{0,
1\}^{\#neurons}$. Each unique activation pattern then defines an activation
region by

\begin{equation}
    \mathcal{R}(\mathcal{A}, \theta) =  \{ x \in \mathbb{R}^{\dim(x)} \; | \; \chi_+(\text{ReLU}(f^l(x_l))) = a_l, \; \forall l, \ldots, k, \; \forall a_l \in \mathcal{A} \}.
\end{equation}

Plotting the activation region for each pixel then shows the different sets of 
activations present in the network. The closer the image looks to noise, the 
more activation patterns are present. In Figure \ref{fig:activation-regions}
we show both the activation regions across the unit square and the
total number of activation regions present across the different encodings. 
We find that the network has similar activation for both the MPEs and higher
frequency FFEs. As such, we believe that the network is still playing a large
role in the architecture, and we should be cautious about reducing its capacity.

\begin{figure}[ht]
    \centering
    \includegraphics[width=0.75\linewidth]{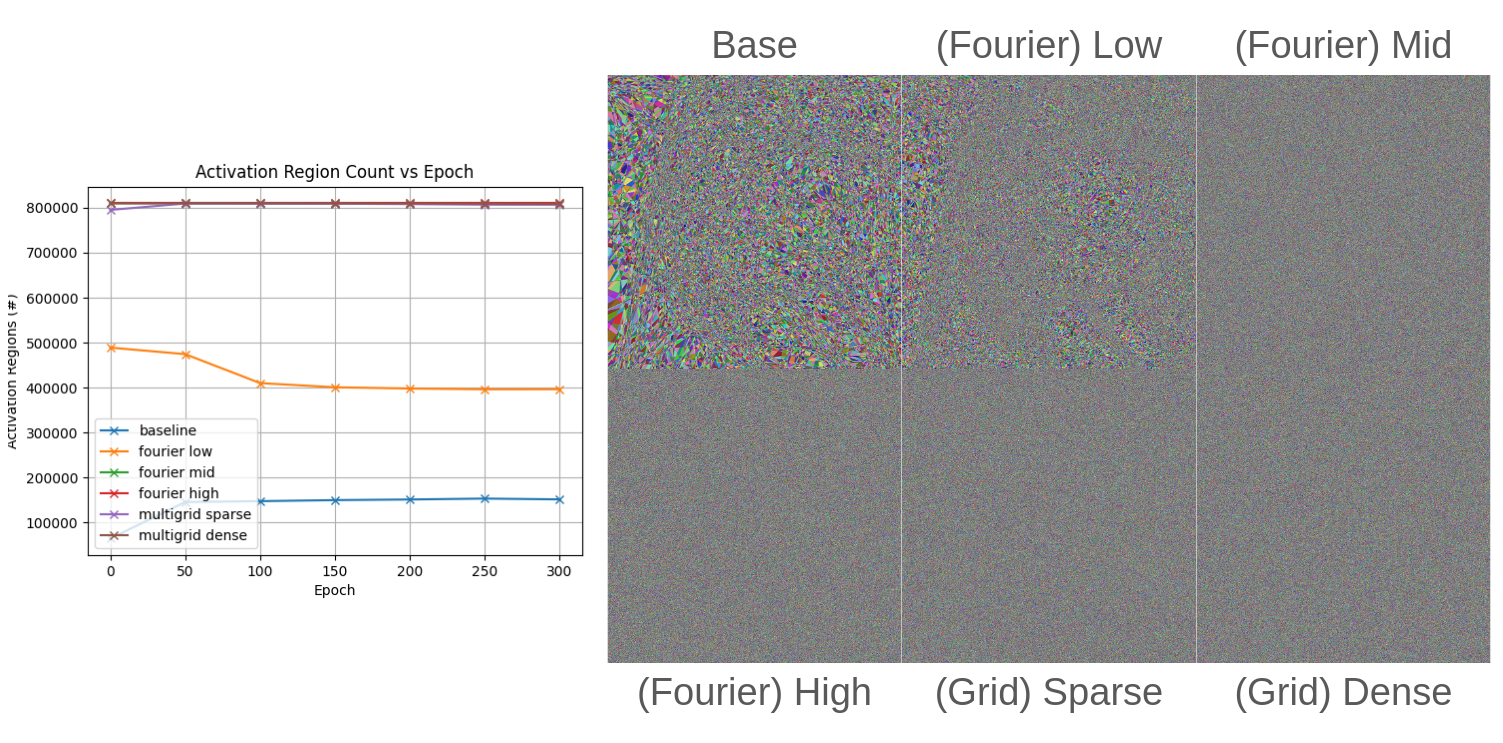}
    \caption{In this figure we plot the total number of activation regions found
    in the network across encodings as well as the activation region plotted per
    pixel of the regressed image. Note that the multigrid parameters are
    slightly different in this instance, with the sparse encoding having 8 layers
    with 2 parameters at each grid cell, while the dense encoding has 24 layers
    with 2 parameters at each grid cell. We find that even though the grid
    encoding seems to be learning the image (Figure
    \ref{fig:multigrid-grayscale}, the baking network still has
    increased activation.}
    \label{fig:activation-regions}
\end{figure}

\section{Hyperparameters}

This section lists the hyperparameters used by encodings in each experiment. Table \ref{tab:tuning-results} shows the hyperparameters used in the ImageNet evaluation. These parameters were found using Optuna \cite{Akiba2019Optuna:Framework} to reflect real world training. The tuner was set to maximize the PSNR value between the network output and ground truth. Table \ref{tab:scaling-params} gives the hand selected parameters that were used to investigate how adjusting the parameters affects the output. We see that increasing the FFE's frequency improves the PSNR, but there are diminishing returns, as there is a smaller difference between Mid and High than there is between Low and Mid. The MPE, in contrast, works well for both the coarse and fine configurations. Lastly, Table \ref{tab:implicit-3d-params} gives the Optuna found encoding parameters for the 3D implicit surface regression. The tuner sought to minimize the training loss. Results for the implicit surface evaluation can be found in Appendix E.

\begin{table}[t]
    \centering
    \begin{tabular}{c|cccc|c}
        & \multicolumn{4}{|c|}{Hyperparameters} & Results \\
        Encoding & Learning Rate & Epochs & Batch Size & Grid Parameters & PSNR \\
        \hline
        Multigrid & 0.3932 & 100 & 10 & $k=3$, $L=2$, $x^\pm=[96,277]$ & 45.28 \\
        Fourier & 0.3865 & 100 & 92 & $L=6$ & 33.47 \\
        Baseline & 0.2394 & 100 & 10 & N/A & 29.94 \\
    \end{tabular}
    \caption{This table shows the results of hyperparameter tuning using the Optuna library. 30 trials were conducted where the tuner was able to select values
    for learning rate, batch size, and encoding specific parameters. The tuner sought to maximize the peak signal-to-noise ratio (PSNR) between the regressed and true image. The multigrid parametric encoding (MPE) is able to maintain a large lead in PSNR over the Fourier feature encoding (FFE) and baseline. These parameters were used in a larger sweep to characterize how the encodings behave across a wide class of images.}
    \label{tab:tuning-results}
\end{table}

\begin{table}[t]
    \centering
    \begin{tabular}{c|cccc|c}
        & \multicolumn{4}{|c|}{Hyperparameters} & Results \\
        Encoding & Learning Rate & Epochs & Batch Size & Grid Parameters & PSNR \\
        \hline
        Low (FFE) & 100 & 300 & 32 & $L=4$ & 25 \\
        Mid (FFE) & 100 & 300 & 32 & $L=8$ & 35 \\
        High (FFE) & 100 & 300 & 32 & $L=16$ & 36.5 \\
        Coarse (MPE) & 100 & 300 & 32 & $k=2, L=1, x=100$ & 34.5 \\
        Fine (MPE) & 100 & 300 & 32 & $k=2, L=1, x=200$ & 41.5 \\
        Baseline & 100 & 300 & 32 & N/A & 20 \\
    \end{tabular}
    \caption{This table shows the parameters used in scaling experiments
    on the different encodings. We compare 3 FFEs of increasing frequency
    and 2 MPEs with fine or coarse grids (e.g., greater than or less than the
    number of pixels in the image). The PSNR column reports the regressed image
    against truth at the end of training. For details on training, please see
    \textit{Experimental Setup}.}
    \label{tab:scaling-params}
\end{table}

\begin{table}[ht]
    \centering
    \begin{tabular}{c|cccc}
        & \multicolumn{4}{|c}{Hyperparameters} \\
        Encoding & Learning Rate & Epochs & Batch Size & Grid Parameters \\
        \hline
        Armadillo (Base) & 0.92224 & 4000 & 13187 & N/A \\
        Armadillo (FFE) & 0.78930 & 4000 & 9799 & $L=7$ \\
        Armadillo (MPE) & 0.99469 & 4000 & 10903 & $k=1,L=1,x=44$ \\
        Buddha (Base) & 0.92224 & 4000 & 13187 & N/A \\
        Buddha (FFE) & 0.78930 & 4000 & 9799 & $L=7$ \\
        Buddha (MPE) & 0.78445 & 4000 & 9267 & $k=2, L=2,x^\pm=[38,102]$ \\
        Dragon (Base) & 0.92224 & 4000 & 13187 & N/A \\
        Dragon (FFE) & 0.78930 & 4000 & 9799 & $L=7$ \\
        Dragon (MPE) & 0.80905 & 4000 & 9989 & $k=2,L=2,x^{\pm}=[33,136]$ \\
    \end{tabular}
    \caption{This table provides encoding parameters used in the evaluation of 3D implicit surfaces from meshes (see Appendix E). Parameters were found using Optuna minimizing the training loss. Grids were restricted to no more than 3 layers deep and 3 trainable parameters per node.}
    \label{tab:implicit-3d-params}
\end{table}

\section{3D Implicit Surface}

The NTK spectrum was evaluated on 3D implicit surface regression to demonstrate that the theory holds in 3D problems as well. The implicit surface was training using an 8 layer MLP with 256 neurons pre hidden layer. The final layer was passed through a sigmoid and then loss was evaluated by binary cross entropy. The input mesh was scaled to normalize its longest axis between 0 and 1. Points were randomly sampled each epoch in the scaled mesh's bounding box and input to the network. An output of 0 meant the point was outside the mesh, and a 1 meant it was inside. To visualize the surface, a ray marching method was used to render the 0.5 level set of the function, which is expected to represent the surface of the mesh. Again, there is no distinct separating between training and test data, and optimization was done using stochastic gradient descent \cite{Kiefer1952StochasticFunction}.

As expected, we see that the MPE outperforms the FFE and baseline in terms of fine detail. The baseline network fails to even learn a recognizable structure in all cases. The FFE is able to represent some structure, but struggles with finer detail. The corresponding average eigenvalue spectrum was plotted in the main paper in Figure \ref{fig:occupancynet-spectra}. The hyperparameters used in each network can be found in Table \ref{tab:implicit-3d-params}.

\begin{figure}
    \centering
    \includegraphics[width=0.75\linewidth]{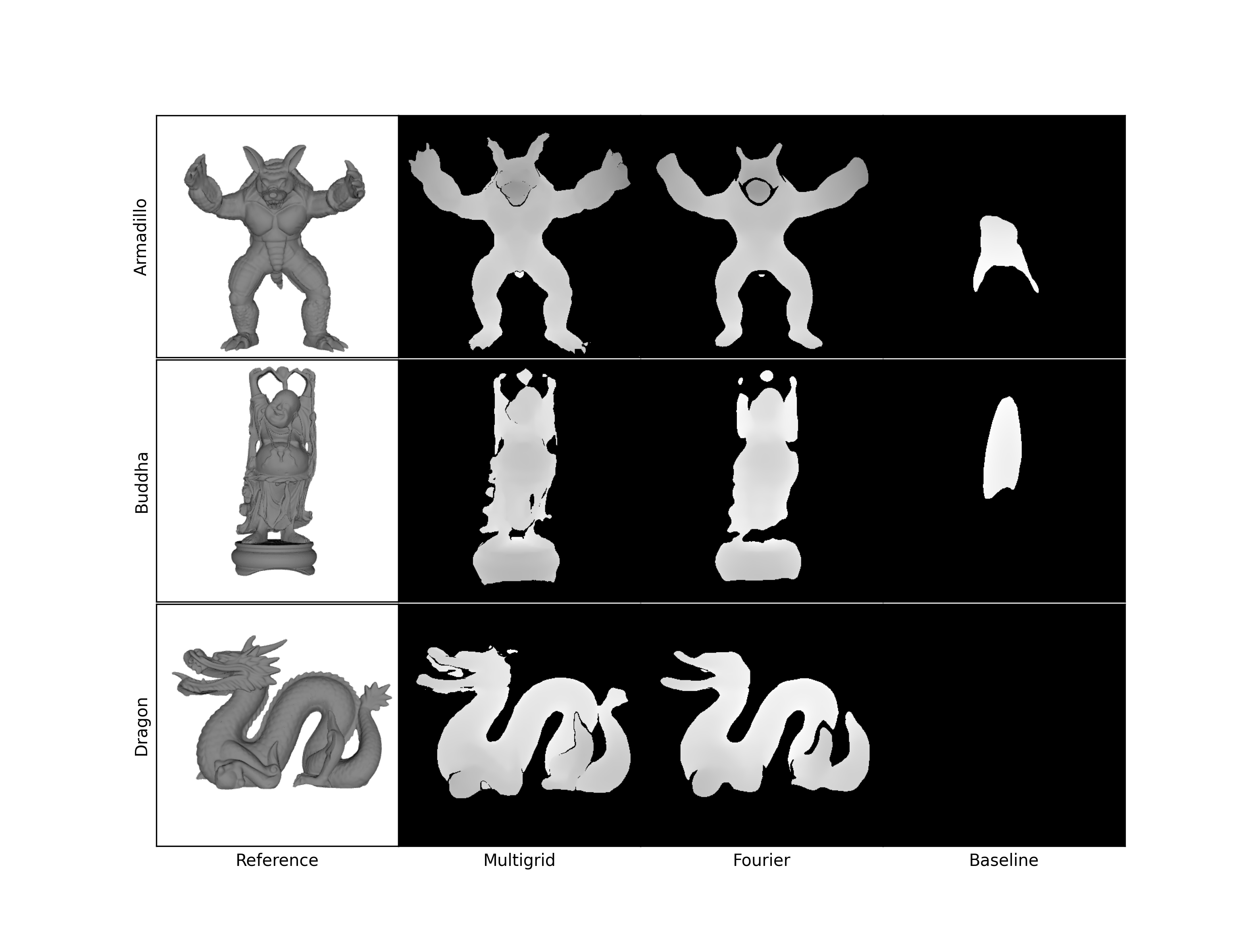}
    \caption{This figure plots the depth field at the 0.5 level set of the learned implicit surface. This rendering shows the network's ability to capture the structure of the input mesh. We find that the MPE outperforms the FFE and baseline in all scenarios, with baseline struggling to capture any recognizable features of the original mesh.}
    \label{fig:3d-implicit-surfaces}
\end{figure}

\end{document}